\documentclass[twocolumn, switch]{article} 

\usepackage{preprint}
\usepackage{multirow}
\usepackage{microtype}
\usepackage{graphicx}
\usepackage{subcaption}
\usepackage{booktabs}
\usepackage{bm}

\usepackage{amsmath}
\usepackage{amssymb}
\usepackage{mathtools}
\usepackage{amsthm}

\theoremstyle{plain}
\newtheorem{theorem}{Theorem}[section]
\newtheorem{proposition}[theorem]{Proposition}

\theoremstyle{definition}

\theoremstyle{remark}
\newtheorem{remark}[theorem]{Remark}
\usepackage{amsmath, amsthm, amssymb, amsfonts}
\RequirePackage{times}

\RequirePackage{fancyhdr}
\RequirePackage{xcolor} 
\RequirePackage{algorithm}
\RequirePackage{algorithmic}
\RequirePackage{eso-pic} 
\RequirePackage{forloop}
\RequirePackage{url}
\RequirePackage{caption}

\usepackage[sort&compress, numbers]{natbib}

\usepackage[utf8]{inputenc}	
\usepackage[T1]{fontenc}	
\usepackage{xcolor}		
\usepackage{hyperref}
\usepackage[capitalize,noabbrev]{cleveref}
\usepackage{booktabs} 		
\usepackage{nicefrac}		
\usepackage{microtype}		
\usepackage{lineno}		
\usepackage{float}			

\usepackage{lipsum}		

\usepackage{newfloat}
\DeclareFloatingEnvironment[name={Supplementary Figure}]{suppfigure}
\usepackage{sidecap}
\sidecaptionvpos{figure}{c}

\usepackage{titlesec}
\title{WFR-MFM: One-Step Inference for Dynamic Unbalanced Optimal Transport}
\date{}
\author{
Xinyu Wang$^{1,\dag}$,\;\;
Ruoyu Wang$^{2,\dag}$,\;\;
Qiangwei Peng$^{1,\dag}$,\;\;
Peijie Zhou$^{2,3,4,5,*}$,\;\;
Tiejun Li$^{1,2,4,6,*}$\\[0.5ex]
\footnotesize
$^1$ LMAM and School of Mathematical Sciences, Peking University. \\
\footnotesize
$^2$ Center for Machine Learning Research, Peking University. \\
\footnotesize
$^3$ Center for Quantitative Biology, Peking University. \\
\footnotesize
$^4$ National Engineering Laboratory for Big Data Analysis and Applications. \\
\footnotesize
$^5$ AI for Science Institute. \\
\footnotesize
$^6$ Center for Data Science, Peking University. \\
\footnotesize
$^\dag$ These authors contribute equally to this work. \\
\footnotesize
$^*$ Correspondence: pjzhou@pku.edu.cn; tieli@pku.edu.cn
}

\begin{document}
\twocolumn[
  \begin{@twocolumnfalse}

\maketitle
\begin{abstract}
Reconstructing dynamical evolution from limited observations is a fundamental challenge in single-cell biology, where dynamic unbalanced optimal transport (OT) provides a principled framework for modeling coupled transport and mass variation.
However, existing approaches rely on trajectory simulation at inference time, making inference a key bottleneck for scalable applications.
In this work, we propose a mean-flow framework for unbalanced flow matching that summarizes both transport and mass-growth dynamics over arbitrary time intervals using mean velocity and mass-growth fields, enabling fast one-step generation without trajectory simulation. 
To solve dynamic unbalanced OT under the Wasserstein--Fisher--Rao geometry, we further build on this framework to develop \textbf{Wasserstein--Fisher--Rao Mean Flow Matching (WFR-MFM)}.
Across synthetic and real single-cell RNA sequencing datasets, WFR-MFM achieves orders-of-magnitude faster inference than a range of existing baselines while maintaining high predictive accuracy, and enables efficient perturbation response prediction on large synthetic datasets with thousands of conditions.
\end{abstract}
\vspace{0.35cm}

  \end{@twocolumnfalse} 
] 



\section{Introduction}
Reconstructing dynamical evolution from limited observations is a fundamental problem across many scientific domains \citep{chen2018neural}. The central challenge lies in inferring a plausible evolution that connects these sparse observations. This challenge is particularly evident in single-cell RNA sequencing (scRNA-seq), where destructive measurements yield independent snapshots at discrete times rather than cell trajectories \citep{zheng2017massively,chen2022spatiotemporal}. Therefore, trajectory inference in scRNA-seq aims to recover the underlying cellular dynamics from such sparse snapshots \citep{waddingot,trajectorynet}. 
A common approach to modeling latent continuous-time dynamics is through neural ordinary differential equations (ODEs) \citep{chen2018neural}, which parameterize system evolution via expressive, learnable vector fields capable of representing complex nonlinear dynamics. 
Despite their flexibility, these models rely on trajectory-level ODE simulation, leading to two major challenges: (i) high computational cost and potential instabilities during \textbf{training} \citep{yan2019robustness}, and (ii) substantial computational overhead at \textbf{inference} due to iterative trajectory integration \citep{song2023consistency,frans2025one,kim2024consistency}.

To alleviate the challenges posed by ODE simulation, recent studies have proposed alternative approaches. 
Flow Matching (FM) \citep{cfm_lipman,liu2022flow,pooladian2023multisample,albergo2022building} reformulates the learning objective to enable simulation-free training by directly learning the velocity field governing an ODE through supervised regression.  
While FM enables simulation-free training, inference remains trajectory-based and still requires ODE simulation.
Some studies \citep{yang2025consistency, peng2025flow} introduce additional consistency constraints when learning instantaneous velocity fields to accelerate inference, while MeanFlow \citep{geng2025mean} parameterizes the mean velocity field over time intervals to enable direct inference within the FM framework without trajectory simulation. 
Such approaches enable fast inference but are primarily formulated under mass-preserving assumptions.

However, in many real-world systems this assumption does not hold, as exemplified by single-cell dynamics involving proliferation and apoptosis.
Several neural ODE-based models have been successfully applied to capture such unbalanced dynamics \citep{TIGON,peng2024stvcr,DeepRUOT,sun2025variational}. 
However, such unbalanced formulations are more computationally demanding than their mass-preserving counterparts, as neural ODE implementations must simulate both transport and mass-growth dynamics during both training and inference.
This has motivated recent work on unbalanced FM, most of which focus on regressing transport velocity fields while neglecting explicit modeling of growth dynamics \citep{eyring2024unbalancedness,corso2025composing,cao2025taming}. 
A principled way to jointly model transport and mass variation is provided by dynamic unbalanced optimal transport (OT), with the Wasserstein--Fisher--Rao (WFR) geometry coupling transport with mass creation and annihilation \citep{liero2016optimal,chizat2018interpolating,kondratyev2016}.
Building on this framework, WFR-FM \citep{peng2026wfrfm} integrates unbalanced OT and FM under the WFR geometry, jointly learning both a transport vector field and a growth-rate function with simulation-free training.
Nevertheless, inference in existing unbalanced FM methods, including WFR-FM, remains trajectory-based and still requires numerical ODE simulation, leaving the inference challenge for unbalanced dynamics largely unresolved.

Efficient inference for unbalanced dynamics has become a central bottleneck in emerging applications such as large-scale single-cell perturbation prediction.
Recent studies have increasingly focused on predicting cellular responses to perturbations from single-cell data at scale \citep{roohani2024predicting,lotfollahi2019scgen,zinati2024groundgan,klein2025cellflow}.
Single-cell perturbations inherently induce unbalanced dynamics, with large changes in cell abundance before and after perturbation, motivating dynamic unbalanced modeling.
Crucially,  the space of perturbation conditions expands combinatorially across modality, dosage, timing, and experimental context such as cell line, tissue, and batch effects \citep{klein2025cellflow}.
Together, these factors make ODE-based inference impractical for perturbation prediction, making fast inference a necessity.

To address this remaining challenge, we introduce a mean-flow framework for unbalanced FM that enables fast one-step generation for mass-varying dynamics by directly advancing system states over finite time intervals, without relying on trajectory-level ODE simulation.
To solve dynamic unbalanced OT under the WFR geometry, we further develop \textbf{Wasserstein--Fisher--Rao Mean Flow Matching (WFR-MFM)} based on the proposed mean-flow framework, yielding a mean-flow matching algorithm with simulation-free training and inference.
We validate the effectiveness of WFR-MFM on a range of synthetic and real scRNA-seq datasets, including a large-scale synthetic perturbation dataset, with systematic comparisons to representative existing methods \citep{TIGON,DeepRUOT,mioflow,sun2025variational,sflowmatch,kapusniak2024metric,rohbeck2025modeling,eyring2024unbalancedness, wang2025joint,peng2026wfrfm}. Our contributions can be summarized as follows:
\begin{itemize}
\item We formulate a mean-flow framework for unbalanced FM, which summarizes transport and mass-growth dynamics over arbitrary time intervals through mean velocity and mass-growth fields. This formulation replaces trajectory-based inference with direct state-to-state updates, enabling rapid one-step generation without simulating continuous-time dynamics.
\item We propose \textbf{WFR-MFM}, a mean flow matching algorithm to approximate dynamic unbalanced OT under the WFR geometry. By learning mean velocity and mass-growth fields, the algorithm enables direct one-step generation along WFR geodesics.
\item We evaluate WFR-MFM using synthetic and real scRNA-seq datasets, comparing its inference efficiency and accuracy with existing methods. Furthermore, we apply WFR-MFM to a large-scale synthetic perturbation dataset to reveal its potential for inference in perturbation response prediction in single-cell biology.
\end{itemize}

\section{Related Work}
\paragraph{Flow Matching.}
Flow matching \citep{cfm_lipman,liu2022flow,pooladian2023multisample,albergo2022building} provides a simulation-free framework for learning continuous velocity fields through conditional probability paths in generative modeling, avoiding  ODE-based path simulation during training. 
The framework has been extended to encompass OT couplings \citep{cfm_tong,klein2024genot}, conditional generation \citep{rohbeck2025modeling}, multi-marginal formulations \citep{rohbeck2025modeling,lee2025mmsfm}, stochastic dynamics \citep{sflowmatch,lee2025mmsfm}, and other generalizations \citep{kapusniak2024metric,zhang2024trajectory,meta_flow_matching,petrovic2025curly}. 
However, most methods assume normalized distributions and thus mass conservation.
In biological systems with proliferation and apoptosis, dynamics are inherently mass-varying, motivating recent work on unbalanced FM \citep{eyring2024unbalancedness,cao2025taming,corso2025composing,wang2025joint,peng2026wfrfm, song2025reweighted}. Among these, WFR-FM \citep{peng2026wfrfm} uniquely provides an OT-consistent formulation that strictly recovers dynamic unbalanced OT under the WFR geometry.

\paragraph{Few-Step Flow Models.}
Few-step generation has been widely explored in diffusion models via consistency-based objectives \citep{song2023consistency}, and similar efficiency concerns have recently driven few-step flow models in FM.
To reduce the inference overhead caused by ODE-based trajectory simulation in FM, few-step or one-step flow models have been proposed for fast generation. The first class of methods \citep{yang2025consistency, peng2025flow, frans2025one} still learns the instantaneous velocity field and accelerates sampling via additional consistency constraints. For example, Consistency-FM \citep{yang2025consistency} enforces self-consistency of the velocity field with a multi-segment training strategy.
More recently, MeanFlow \citep{geng2025mean} departs from instantaneous dynamics by parameterizing velocity fields averaged over finite time intervals, enabling direct one-step inference within the FM framework. Subsequent works propose various refinements to the MeanFlow framework \citep{guo2025splitmeanflow,zhang2025alphaflow,geng2025improved}. 
Beyond averaged velocity parameterizations, closely related solution-based formulations target one-step generation from scratch by learning the solution function of the velocity ODE, as exemplified by SoFlow \citep{luo2025soflow}. However, these approaches are fundamentally built on mass-preserving assumptions and are not applicable to unbalanced dynamics.

\paragraph{Unbalanced Optimal Transport and Dynamic Models of scRNA-seq Data.}
Unbalanced optimal transport generalizes classical OT \citep{kantorovich1942translocation,benamou2000computational} by allowing mass variation, which is essential for modeling biological systems involving proliferation and apoptosis. It relaxes the hard marginal constraints of classical OT \citep{benamou2003numerical,figalli2010optimal,caffarelli2010free} by penalizing deviations of transported mass through suitable divergence terms. The WFR metric \citep{liero2016optimal,chizat2018interpolating,kondratyev2016} provides a dynamic formulation of unbalanced OT that jointly models transport and mass variation through coupled velocity and mass-growth fields. Recently, deep learning solvers for dynamic WFR have been developed to characterize continuous-time single-cell dynamics \citep{action_matching,tong_action,TIGON,peng2024stvcr,DeepRUOT,sun2025variational,peng2026wfrfm}. Despite providing OT-consistent descriptions of unbalanced dynamics, existing methods typically rely on trajectory-based simulation at inference time rather than direct one-step generation, leading to substantial computational overhead at scale.
Motivated by this limitation, we develop \textbf{WFR-MFM} under the WFR geometry as an efficient method for inference in dynamic unbalanced optimal transport.

\section{Mathematical Background}
\paragraph{Setup.} We consider two endpoint measures $\mu_0$ and $\mu_1$ supported on a domain $\mathcal{X} \subset \mathbb{R}^d$, with density functions $\mu_0(\bm{x})$ at $t = 0$ and $\mu_1(\bm{x})$ at $t = 1$. Let $\mathcal{M}_+(\mathcal{X})$ denote the space of finite, absolutely continuous nonnegative measures on $\mathcal{X}$. The total masses of $\mu_0$ and $\mu_1$ are not required to match.

\paragraph{Flow Matching for Probability Flow.}
FM \citep{cfm_lipman,liu2022flow,pooladian2023multisample,albergo2022building} provides a general framework for learning continuous-time dynamics of measures by regressing parameterized vector fields to analytically specified targets.
In the mass-preserving setting, FM models a time-dependent probability path $\{\rho_t\}_{t\in[0,1]}$ with velocity field $\bm{u}_t$ satisfying the continuity equation
\begin{equation*}
\partial_t \rho_t + \nabla \!\cdot (\rho_t \bm{u}_t) = 0,
\end{equation*}
which characterizes deterministic measure flows induced by transport. If a kinetic energy minimization over admissible transport paths is imposed, the formulation reduces to the dynamic OT problem \citep{benamou2000computational}.
FM learns such dynamics by regressing a parameterized vector field $\bm{u}_\theta(\bm{x},t)$ to the true field, avoiding backpropagation through ODE solvers:
\begin{equation*}
    \mathcal{L}_{\mathrm{FM}}(\theta)
    = \mathbb{E}_{t,\bm{x} \sim \rho_t}
      \big[\|\bm{u}_\theta(\bm{x},t) - \bm{u}_t(\bm{x})\|_2^2\big].    
\end{equation*}
To obtain a tractable training objective, FM constructs conditional paths $\rho_t(\cdot\mid\bm{z})$ admitting closed-form vector fields $\bm{u}_t(\cdot\mid\bm{z})$, and defines the regression objective
\begin{equation*}
    \mathcal{L}_{\mathrm{CFM}}(\theta)
    = \mathbb{E}_{t,\bm{z},\bm{x} \sim \rho_t(\cdot\mid \bm{z})}
    \big[\|\bm{u}_\theta(\bm{x},t) - \bm{u}_t(\bm{x}\mid\bm{z})\|_2^2\big].
\end{equation*}
When the conditional paths and coupling distribution are chosen to match the classical OT formulation, the learned dynamics recover the geodesics of the dynamic OT problem \citep{cfm_tong}.

\paragraph{Variable mass flow.} Consider a measure path : $[0,1] \times \mathbb{R}^d \to \mathbb{R}_+$, together with a velocity field $\bm{u}_t : [0,1] \times \mathbb{R}^d \to \mathbb{R}^d$ and a growth-rate function $g_t : [0,1] \times \mathbb{R}^d \to \mathbb{R}$. Assuming sufficient regularity, $(\bm{u}_t, g_t)$ induces a unique time-dependent flow $\bm{\phi}_t : [0,1] \times \mathbb{R}^d \to \mathbb{R}^d$, $m_t : [0,1] \times \mathbb{R}^d \to \mathbb{R}$ defined via an ODE:   
\begin{equation*}\label{eq:flow}
\begin{gathered}
	\frac{d}{dt}
\begin{pmatrix}
\bm{\phi}_t(\bm{x}) \\[4pt]
\ln m_t(\bm{x})
\end{pmatrix}
=
\begin{pmatrix}
\bm{u}_t(\bm{\phi}_t(\bm{x})) \\[4pt]
g_t(\bm{\phi}_t(\bm{x}))
\end{pmatrix},\\
\bm{\phi}_0(\bm{x}) = \bm{x},\;\;\; m_0(\bm{x})\; \text{given},
\end{gathered}
\end{equation*}
which defines a push-forward mapping with changing mass $(\bm{\phi}_t,m_t)_{\#} : \mathcal{M_+}(\mathbb{R}^d) \to \mathcal{M}_+(\mathbb{R}^d)$ for each $t\in [0,1]$ such that
\begin{equation*}
\begin{aligned}
\rho_t(\bm{x}) &\ = (\bm{\phi}_t,m_t)_{\#} \rho_0(\bm{x}) \\
&:= \rho_0(\bm{\phi}_t^{-1}(\bm{x})) \left| \det \nabla_{\bm{x}} \bm{\phi}_t^{-1}(\bm{x}) \right|
\frac{m_t(\bm{x})}{m_0(\bm{\phi}_t^{-1}(\bm{x}))}	
\end{aligned}
\end{equation*}
and satisfies the continuity equation with source term
\begin{equation*}
\label{marginal continuity equation with source term}
\partial_t\rho_t(\bm{x})+\nabla_{\bm{x}} \cdot (\bm{u}_t(\bm{x})\rho_t(\bm{x})) = g_t(\bm{x})\rho_t(\bm{x}). 
\end{equation*}
Furthermore, imposing a variational principle penalizing transport and mass variation yields dynamic unbalanced  OT problems. A prominent example is the WFR metric \citep{chizat2018interpolating,chizat2018unbalanced,liero2018optimal}, which provides a principled formulation for the dynamic evolution $\{\rho_t\}_{t \in [0,1]}$ from $\mu_0$ to $\mu_1$ while jointly modeling transport and mass variation:
\begin{equation}\label{eq:wfr_dynamic}
\begin{gathered}
\begin{aligned}
\mathrm{WFR}_{\delta}^2(\mu_0,\mu_1)
=\ &\inf_{\rho,\bm{u},g} \int_0^1\!\int_{\mathcal{X}}
\tfrac{1}{2}\big(\|\bm{u}(\bm{x},t)\|_2^2 \\
&+ \delta^2 \|g(\bm{x},t)\|_2^2\big)
\rho_t(\bm{x})\,\mathrm{d}\bm{x}\mathrm{d}t,	
\end{aligned}\\
\text{s.t. }\;
\partial_t\rho+\nabla_{\bm{x}}\!\cdot(\rho \bm{u}) = \rho g,\;\;\;
\rho_0 = \mu_0,\ \rho_1 = \mu_1.
\end{gathered}
\end{equation}

\paragraph{Unbalanced Flow Matching.}
Unbalanced FM extends FM to variable-mass dynamics by jointly learning a velocity field $\bm{u}_t$ and a growth-rate function $g_t$ via neural networks $\bm{u}_{\bm{\theta}}(\bm{x},t)$ and $g_{\bm{\phi}}(\bm{x},t)$ trained under the unbalanced FM objective
\begin{equation*}
\begin{aligned}
\mathcal{L}_{\mathrm{UFM}}(\bm{\theta},\bm{\phi})
= \ &\mathbb{E}_{t,\bm{x} \sim \rho_t}
\Big(
\|\bm{u}_{\bm{\theta}}(\bm{x},t) - \bm{u}_t(\bm{x})\|_2^2 \\
&+ \lambda\,\| g_{\bm{\phi}}(\bm{x},t) - g_t(\bm{x}) \|_2^2
\Big).	
\end{aligned}
\label{eq:UFM}
\end{equation*}
To obtain a tractable training objective, unbalanced FM adopts a latent-variable formulation with
$\bm{z}=((\bm{x}_0,m_0),(\bm{x}_1,m_1))$ sampled from a coupling distribution $q(\bm{z})$, and constructs conditional measure paths $\rho_t(\bm{x}\mid \bm{z})$ together with corresponding conditional target fields, leading to the conditional unbalanced flow matching (CUFM) objective
\begin{equation*}
\begin{aligned}
\mathcal{L}_{\text{CUFM}}(\bm{\theta},\bm{\phi})
=\ &\mathbb{E}_{t,\, \bm{z},\, \bm{x}\sim \rho_t(\bm{x}\vert \bm{z})}
\Big(
\|\bm{u}_{\bm{\theta}}(\bm{x},t)
-\bm{u}_t(\bm{x}\vert \bm{z})\|_2^2\\
&+ \lambda \| g_{\bm{\phi}}(\bm{x},t)- g_t(\bm{x}\vert \bm{z}) \|_2^2
\Big).
\label{eq:CUFM}
\end{aligned}
\end{equation*}	
With appropriate conditional paths and couplings, unbalanced FM recovers the geodesics of dynamic unbalanced OT; for example, under the WFR geometry, it leads to the WFR-FM formulation \citep{peng2026wfrfm}. Nevertheless, existing unbalanced FM methods require simulating the full state trajectory $(x_t,m_t)_{t\in[0,1]}$ during inference in order to obtain the terminal state $(x_1,m_1)$, which incurs substantial computational cost at scale.
We next introduce a mean-flow framework for unbalanced FM that enables direct one-step inference.

\section{Mean-Flow for Unbalanced Dynamic OT}
In this section, we introduce the concept of \emph{mean flow} and formulate a general framework for one-step generation in unbalanced FM.
We then specialize this framework to solve dynamic unbalanced OT under the WFR geometry, resulting in WFR-MFM.
\subsection{Mean-flow Variables and Derivative Identities}\label{sec:4.1}
\paragraph{Mean-flow Variables.} To this end, we introduce time-averaged dynamics through the \emph{mean velocity} $\bm{v}(\bm{x}, t, T)$ and \emph{mean mass-growth rate} $h(\bm{x}, t, T)$. For any trajectory $\bm{x}_\tau$ over a time interval $(t, T)$ with $t < T$, we define
\begin{equation}
\begin{aligned}
&\bm{v}(\bm{x}, t, T)
\;=\;
\frac{1}{T - t}\int_t^T \bm{u}_\tau(\bm{x}_\tau)\,d\tau,\\
&h(\bm{x}, t, T)
\;=\;
\frac{1}{T - t}\int_t^T g_\tau(\bm{x}_\tau)\,d\tau,  
\end{aligned}
\label{eq:mean_fields}
\end{equation}
where $\bm{u}_\tau$ and $g_\tau$ denote the transport velocity and mass-growth rate, respectively. These mean-flow variables represent the mean velocity and mass-growth rate required to evolve the system from $(\bm{x}_t, m_t)$ at time $t$ to $(\bm{x}_T, m_T)$ at time $T$ under the flow
$$
\frac{\mathrm{d} \bm{x_\tau}}{\mathrm{d} \tau} = \bm{u}_\tau(\bm{x}_\tau),\;\;\; \bm{x}_t = \bm{x}.
$$
By definition, the mean-flow variables $\bm{v}$ and $h$ satisfy appropriate boundary conditions and inherent consistency constraints. As the time interval shrinks, they recover the instantaneous quantities:
\begin{equation*}
\lim_{T \to t} \bm{v}(\bm{x}, t, T) = \bm{u}(\bm{x}, t),
\;\;\;
\lim_{T \to t} h(\bm{x}, t, T) = g(\bm{x}, t).
\label{eq:limit_fields}
\end{equation*}
Equivalently, the definitions of $\bm{v}$ and $h$ can be rewritten as
\begin{equation}
\begin{gathered}
(T - t)\,\bm{v}(\bm{x}, t, T)
\;=\;
\int_t^T \bm{u}_\tau(\bm{x}_\tau)\,d\tau,\\
(T - t)\,h(\bm{x}, t, T)
\;=\;
\int_t^T g_\tau(\bm{x}_\tau)\,d\tau.
\end{gathered}
\label{eq:integral_form}
\end{equation}

\paragraph{Additive consistency.}
The additivity of the integral implies, for any intermediate $s$ with $t < s < T$,
\[
(T-t)\,\bm{v}(\bm{x},t,T)
= (s-t)\,\bm{v}(\bm{x},t,s)
+ (T-s)\,\bm{v}(\bm{x}_s,s,T).
\]
Therefore, $\bm{v}$ satisfies multi-step consistency over any partition
$t_0 = t < t_1 < \cdots < t_K = T$:
\[
(T-t)\,\bm{v}(\bm{x},t,T)
= \sum_{k=0}^{K-1} (t_{k+1}-t_k)\,\bm{v}(\bm{x}_{t_k}, t_k, t_{k+1}).
\]
The mean mass-growth rate field $h(\bm{x},t,T)$ satisfies the same additive and multi-step consistency.

\paragraph{Derivative identities.}
Differentiating Eq.~\eqref{eq:integral_form} with respect to $t$ yields
\[
\begin{gathered}
 \bm{v}(\bm{x}, t, T)
= \bm{u}_t(\bm{x}) + (T-t)\,\frac{d}{dt}\bm{v}(\bm{x},t,T),\\
h(\bm{x}, t, T)
= g_t(\bm{x}) + (T-t)\,\frac{d}{dt}h(\bm{x},t,T),
  \end{gathered}
\]
where
\[
\begin{gathered}
 \frac{d}{dt}\bm{v}(\bm{x},t,T)
= \partial_t \bm{v}(\bm{x},t,T)
  + (\nabla_{\bm{x}} \bm{v}(\bm{x},t,T))\,\bm{u}_t(\bm{x}),\\  
\frac{d}{dt}h(\bm{x},t,T)
= \partial_t h(\bm{x},t,T)
  + \nabla_{\bm{x}} h(\bm{x},t,T)\cdot \bm{u}_t(\bm{x}).   
\end{gathered}
\]
These derivative identities link the mean-flow variables $(\bm{v},h)$ to the instantaneous fields $(\bm{u}_t,g_t)$.

\subsection{Inference and Training}
\paragraph{Inference.}
Given the mean-flow variables $(\bm{v},h)$, the mean-flow framework provides a direct terminal mapping for unbalanced dynamics. The \textbf{one-step update}
\[
\bm{x}_1 = \bm{x}_0 + \bm{v}(\bm{x}_0, 0, 1),
\;\;\;
m_1 = m_0 \,\exp\!\big(h(\bm{x}_0, 0, 1)\big)
\]
computes the terminal state in a single evaluation.
More generally, inference can be performed via a \textbf{multi-step update} based on additive consistency. For a partition
$0 = t_0 < t_1 < \cdots < t_K = 1$:
\begin{equation*}
\begin{gathered}
\bm{x}_{t_{k+1}} = \bm{x}_{t_k} + (t_{k+1}-t_k)\, \bm{v}(\bm{x}_{t_k}, t_k, t_{k+1}),\\
m_{t_{k+1}} = m_{t_k}\,\exp\!\big((t_{k+1}-t_k)\, h(\bm{x}_{t_k}, t_k, t_{k+1})\big).			
\end{gathered}
\end{equation*}
When intermediate time points are available, the partition follows the data; otherwise it can be chosen arbitrarily.
One-step inference corresponds to $K=1$, while larger $K$ provides a finer approximation to the underlying continuous dynamics. The complete inference procedure is summarized in Algorithm~\ref{alg:WFR-MFM-inference}.
\begin{algorithm}[t]
\caption{Mean-Flow Inference for Unbalanced FM}
  \label{alg:WFR-MFM-inference}
  \begin{algorithmic}

    \STATE {\bfseries Input:}
    mean-flow variables $\bm{v}$ and $h$; initial state $(\bm{x}_0, m_0)$; time partition $0=t_0<\cdots<t_K=1$ (optional).

    \STATE {\bfseries Output:}
    Terminal state $(\bm{x}_1,m_1)$.
    
        \STATE $(\bm{x}_{t_0}, m_{t_0}) \gets (\bm{x}_0, m_0)$
        \FOR{$k = 0$ {\bfseries to} $K-1$}
            \STATE $\bm{x}_{t_{k+1}} \gets \bm{x}_{t_k}+ (t_{k+1}-t_k)\, \bm{v}(\bm{x}_{t_k}, t_k, t_{k+1})$
            \STATE $m_{t_{k+1}} \gets m_{t_k}\exp\!\big((t_{k+1}-t_k)\,h(\bm{x}_{t_k}, t_k, t_{k+1})\big)$
        \ENDFOR
        \STATE {\bfseries return} $(\bm{x}_{1}, m_{1})$
  \end{algorithmic}
\end{algorithm}

\paragraph{Training.}
We train the mean-flow variables $(\bm v, h)$ within the mean-flow framework by parameterizing them with neural networks and enforcing the derivative identities in Sec.~\ref{sec:4.1} via the \textbf{unbalanced mean flow matching loss}:
\begin{equation*}
\begin{aligned}
\mathcal{L}(\bm{\theta},\bm{\phi})
=&\  \mathbb{E}_{t<T,\;\bm{x}\sim\rho_t(\bm{x})}
\Big[
\| \bm{v}_{\bm{\theta}}(\bm{x},t,T) - \text{sg}(\bm{v}(\bm{x},t,T))\|_2^2\\
&+ \lambda\,\| h_{\bm{\phi}}(\bm{x},t,T) - \text{sg}(h(\bm{x},t,T))\|_2^2
\Big],		
\end{aligned}
\end{equation*} 
where $\lambda>0$ balances the supervision of the velocity and mass-growth fields, and $\text{sg}(\cdot)$ is the stop-gradient operator. 
The regression targets $(\bm{v}, h)$ are computed via the derivative identities
\begin{equation*}
\begin{aligned}
\bm{v}(\bm{x}, t, T )
= \ & \bm{u}_t(\bm{x} )
+ (T - t)\!\Bigr[
\partial_t \bm{v}_{\bm{\theta}}(\bm{x},t,T)\\
&+ (\nabla_{\bm{x}} \bm{v}_{\bm{\theta}}(\bm{x},t,T))\,\bm{u}_t(\bm{x} )
\Bigr],\\[0.4em]
h(\bm{x}, t, T )
=\ & g_t(\bm{x} )
+ (T - t)\!\Bigr[
\partial_t h_{\bm{\phi}}(\bm{x},t,T)\\
&+ \nabla_{\bm{x}} h_{\bm{\phi}}(\bm{x},t,T)\cdot \bm{u}_t(\bm{x})
\Bigr].
\end{aligned}	
\end{equation*}
Minimizing the regression loss enforces $(\bm{v}_{\bm{\theta}},h_{\bm{\phi}})$ to recover the true mean-flow variables. 
To obtain a tractable training objective, we introduce conditional mean-field targets 
$\bm{v}(\bm{x},t,T\mid\bm{z})$ and $h(\bm{x},t,T\mid\bm{z})$, 
leading to the \textbf{conditional unbalanced mean flow matching loss}:
\begin{equation*}
\begin{aligned}
\mathcal{L}_{\mathrm{c}}(\bm{\theta},\bm{\phi}) =\ &  \mathbb{E}_{(t,T)\,:\,t<T,\bm{z}, \bm{x}\sim \rho_t(\bm{x}\vert \bm{z})}
\Big[
\| \bm{v}_{\bm{\theta}}(\bm{x},t,T) \\
&- \text{sg}(\bm{v}(\bm{x},t,T\mid\bm{z}))\|_2^2+ \lambda\,\| h_{\bm{\phi}}(\bm{x},t,T) \\
&- \text{sg}(h(\bm{x},t,T\mid\bm{z}))\|_2^2
\Big],		
\end{aligned}
\end{equation*}
where the conditional regression targets are defined as
\begin{equation*}
\begin{aligned}
\bm{v}(\bm{x}, t, T \mid \bm{z})
=\ & \bm{u}_t(\bm{x} \mid \bm{z})
+ (T - t)\,\Bigl[
\partial_t \bm{v}_{\bm{\theta}}(\bm{x},t,T) \\
&
+ (\nabla_{\bm{x}} \bm{v}_{\bm{\theta}}(\bm{x},t,T))\,\bm{u}_t(\bm{x} \mid \bm{z})
\Bigr],\\[0.4em]
h(\bm{x}, t, T \mid \bm{z})
= \ & g_t(\bm{x} \mid \bm{z})
+ (T - t)\,\Bigl[
\partial_t h_{\bm{\phi}}(\bm{x},t,T) \\
&
+ \nabla_{\bm{x}} h_{\bm{\phi}}(\bm{x},t,T)\cdot \bm{u}_t(\bm{x} \mid \bm{z})
\Bigr].
\end{aligned}
\end{equation*}
Here $\bm{u}_t(\bm{x}\mid\bm{z})$ and $g_t(\bm{x}\mid\bm{z})$ represent the instantaneous transport velocity and the mass-growth rate associated with the conditional measure path.

The following theorem shows that the unconditional objective differs from its conditional counterpart only by a constant, thereby justifying the optimization of the tractable conditional loss $\mathcal{L}_{\mathrm{c}}$. The proof is provided in Appendix~\ref{proof: thm1}.
\begin{theorem}\label{thm1}
If $\rho_t(\bm{x})>0$ for all $\bm{x}\in\mathcal{X}$ and $q(\bm{z})$ is independent of $(\bm{x},T)$, then
$
\mathcal{L}(\bm{\theta},\bm{\phi})
= \mathcal{L}_{\mathrm{c}}(\bm{\theta},\bm{\phi}) + C,
$
for a constant $C$ that does not depend on $(\bm{\theta},\bm{\phi})$. Consequently,
\[
\nabla_{(\bm{\theta},\bm{\phi})}\,\mathcal{L}(\bm{\theta},\bm{\phi})
= \nabla_{(\bm{\theta},\bm{\phi})}\,\mathcal{L}_{\mathrm{c}}(\bm{\theta},\bm{\phi}).
\]
\end{theorem}
Thus, the conditional and unconditional objectives yield identical optimal estimators of the mean-flow variables, and the method naturally extends from two to multiple time points. 

\begin{remark}
\label{rem:conditional_mean_fields}
The choice of the mean-flow targets and their conditional counterparts in the unbalanced setting is \textit{not rivial}.
Their specific construction via the mean-flow derivative identities is crucial for the validity of Theorem~\ref{thm1}.
A seemingly natural alternative is to define the learning objective directly according to the mean-flow definition in Eq.~\eqref{eq:mean_fields}.
The corresponding conditional mean-flow variables are given by
\begin{equation*}
\begin{aligned}
&\bm{\tilde{v}}(\bm{x}, t, T\mid\bm{z})
\;=\;
\frac{1}{T - t}\int_t^T \bm{u}_\tau(\bm{x}_\tau\mid\bm{z})\,d\tau,\\
&\tilde{h}(\bm{x}, t, T\mid\bm{z})
\;=\;
\frac{1}{T - t}\int_t^T g_\tau(\bm{x}_\tau\mid\bm{z})\,d\tau.
\end{aligned}
\end{equation*}
However, such constructions generally break the equivalence between the conditional and unconditional objectives, whose proof is provided in Appendix~\ref{proof: thm2}.
\end{remark}

\paragraph{WFR-MFM for Dynamic Unbalanced OT.}
Within the general mean-flow framework for unbalanced FM, the choice of the coupling distribution $q(z)$ determines the underlying transport geometry.
When $q(z)$ is chosen as the unbalanced OT coupling induced by an optimal entropy transport problem under the WFR geometry, the resulting mean-flow dynamics follow WFR geodesics and admit direct state-to-state evolution. Theoretical details on the construction of $q(z)$ are provided in Appendix~\ref{Appendix: WFR-MFM}.

This specialization gives rise to WFR-MFM, a mean-flow algorithm for dynamic unbalanced OT that enables one-step inference under the WFR geometry without trajectory-level ODE simulation.
The training algorithm is summarized in Algorithm~\ref{alg:WFR-MFM}.

\begin{algorithm}[t]
  \caption{WFR-MFM Training}
  \label{alg:WFR-MFM}
  \begin{algorithmic}
    \STATE {\bfseries Input:} Sample-able distributions $\mu_0,\mu_1$; bandwidth $\sigma$; OET batch size $B$; training batch size $b$; WFR penalty $\delta$; mean-flow variables $\bm{v}_{\bm{\theta}}(\bm{x},t,T)$ and $h_{\bm{\phi}}(\bm{x},t,T)$
    \STATE {\bfseries Output:} Trained mean fields $\bm{v}_{\bm{\theta}}$ and $h_{\bm{\phi}}$
    \FOR{$k = 0 \to K-1$}
        \STATE $\gamma^{(k)} \gets \text{OET}(\mu_{t_k}, \mu_{t_{k+1}})$ {\bfseries or} mini-batch $\mathrm{OET}(\mu_{t_k}, \mu_{t_{k+1}};B)$
        \STATE $\gamma_0^{(k)}(\bm{x},\bm{y}) \gets \frac{\gamma^{(k)}(\bm{x},\bm{y})}{\int_\mathcal{X}\gamma^{(k)}{(\bm{x},\bm{z})\mathrm{d} \bm{z}}}\mu_{t_k}(\bm{x})$
    \ENDFOR

    \WHILE{Training}
      \FOR{$k = 0 \to K-1$}
              \STATE Sample $b$ pairs $(\bm{x}_{t_k}, \bm{x}_{t_{k+1}}) \sim \gamma_0$
              \STATE Compute travelling-Dirac constants $A, B, \bm{\omega}_0, \tau$
              \STATE Sample time pairs $(t,T)\in [t_k,t_{k+1}]$ with $t<T$
              \STATE $\bm{\eta}_{t^{(k)}} \gets \bm{x}_{t_k} +\bm{\omega}_0 \Lambda_t(\bm{x}_{t_k},\bm{x}_{t_{k+1}}) $
              \STATE Sample $\bm{x}^{(k)} \sim \mathcal{N}(\bm{\eta}_{t^{(k)}}, \sigma^2 \mathbf{I})$
              \STATE $\bm{u}^{(k)} \gets \bm{\omega}_0/(m_{t^{}}(\bm{x}_{t_k},\bm{x}_{t_{k+1}})(t_{k+1}-t_k))$
              \STATE $g^{(k)} \gets \frac{\mathrm{d}}{\mathrm{d}t} \ln m_{t}(\bm{x}_{t_k},\bm{x}_{t_{k+1}})/(t_{k+1}-t_k)$
              \STATE $m^{(k)} \gets m_{t}(\bm{x}_{t_k},\bm{x}_{t_{k+1}}) / \gamma_0^{(k)}(\bm{x}_{t_k},\bm{x}_{t_{k+1}})$
              \STATE $\bm{v}^{(k)}  \gets \bm{u}^{(k)} + \text{sg} ((T-t)\,[\partial_t \bm{v}_{\bm{\theta}}(\bm{x},t,T) + (\nabla_{\bm{x}} \bm{v}_{\bm{\theta}}(\bm{x},t,T))\,\bm{u}^{(k)}])$
              \STATE $h^{(k)} \gets g^{(k)} + \text{sg} ((T-t)\,[\partial_t h_{\bm{\phi}}(\bm{x},t,T) + \nabla_{\bm{x}} h_{\bm{\phi}}(\bm{x},t,T)\cdot\bm{u}^{(k)}])$
      \ENDFOR
      \STATE Concatenate $\{\bm{x}^{(i)}, t^{(i)}, T^{(i)},\bm{v}^{(i)}, h^{(i)}, m^{(i)}\}_{i=1}^K$ into batched tensors $\{\bm{x}^b, t^b, T^b, \bm{v}^b, h^b, m^b\}$ 
      \STATE $\mathcal{L}_{\mathrm c}(\bm{\theta},\bm{\phi}) \gets \big(\|\bm{v}_{\bm{\theta}}(\bm{x}^b,t^b,T^b)-\bm{v}^b\|_2^2 + \lambda\,\|h_{\bm{\phi}}(\bm{x}^b,t^b,T^b)-h^b\|_2^2\big)\,m^b$
      \STATE $(\bm{\theta},\bm{\phi}) \gets \mathrm{Update}\big((\bm{\theta},\bm{\phi}),(\nabla_{\bm{\theta}}\mathcal{L}_{\mathrm c},\nabla_{\bm{\phi}}\mathcal{L}_{\mathrm c})\big)$
    \ENDWHILE

    \STATE {\bfseries return} $\bm{v}_{\bm{\theta}}$ and $h_{\bm{\phi}}$.
  \end{algorithmic}
\end{algorithm}

\section{Numerical Results}
We evaluate WFR-MFM by addressing the following questions:
\textbf{Q1}: Does WFR-MFM significantly accelerate inference?
\textbf{Q2}: Does WFR\text{-}MFM maintain high predictive accuracy under fast inference?
\textbf{Q3}: Is there a controllable trade-off between inference speed and accuracy?
\textbf{Q4}: Is WFR-MFM scalable in both training and inference?

\paragraph{Massive acceleration in inference speed (Q1).}
We evaluate inference runtime on three synthetic datasets (Gene, Dyngen, Gaussian), the real-world embryoid bodies (EB) dataset \citep{moon2019visualizing}, and three real multi-time scRNA-seq datasets (EMT \citep{cook2020context}, CITE-seq \citep{lance2022multimodal}, and Mouse hematopoiesis \citep{weinreb2020lineage}). 
Existing methods for dynamic unbalanced OT perform inference via ODE simulation and thus incur similar computational costs. As a representative baseline, we compared against WFR-FM, which follows the standard ODE-based inference paradigm. 
Its original implementation uses an adaptive Dormand–Prince RK5(4) solver \citep{dormand1980family}; we additionally report results with a fixed-step explicit Euler solver (100 steps). The runtime comparison is presented in Table~\ref{tab:inference_time}. Across all datasets, WFR-MFM achieves substantial acceleration relative to these ODE-based baselines, providing inference \textit{speedups of two to three orders of magnitude} compared to adaptive RK solvers and remaining approximately \textit{two orders of magnitude faster} than 100-step Euler solvers. These results demonstrate the significant advantages of WFR-MFM in fast inference.

\begin{table*}[t]
\caption{\textbf{Inference time comparison across datasets.}
All values are reported as \textit{mean $\pm$ sd} (seconds per inference). $N_0$ denotes the number of cells at the initial time point and $T$ is the number of future time points to infer. WFR-MFM is evaluated over 1000 runs, while WFR-FM with adaptive Dormand--Prince RK5(4) and with fixed-step Euler (100 steps) are each evaluated over 100 runs. \textit{Spd-RK} and \textit{Spd-Eu} denote the speedup of WFR-MFM relative to WFR-FM (RK) and WFR-FM (Euler), respectively.
\textit{Spd-RK} and \textit{Spd-Eu} denote the speedup of WFR-MFM relative to
WFR-FM (RK) and WFR-FM (Euler), respectively.}
\label{tab:inference_time}

  \begin{center}
    \begin{small}
      \begin{sc}
             \resizebox{0.95\textwidth}{!}{
        \begin{tabular}{lcccccccc}
          \toprule
          \textbf{Dataset} & \textbf{Dim} & $\mathbf{N_0}$ & $\mathbf{T}$ &
          \textbf{WFR-FM (RK)} & \textbf{WFR-FM (Euler)} & \textbf{WFR-MFM} &
          \textbf{Spd-RK} & \textbf{Spd-Eu} \\
          \midrule
          Gene     & 2    & 400  & 4 & 2.178 \tiny$\pm$0.148 & 0.370 \tiny$\pm$0.021 & 0.0038 \tiny$\pm$0.0009 & 573$\times$  & 97$\times$  \\
          Dyngen   & 5    & 156  & 4 & 0.488 \tiny$\pm$0.049 & 0.409 \tiny$\pm$0.039 & 0.0039 \tiny$\pm$0.0010 & 125$\times$  & 105$\times$ \\
          Gaussian & 1000 & 500  & 1 & 0.794 \tiny$\pm$0.087 & 0.143 \tiny$\pm$0.018 & 0.0013 \tiny$\pm$0.0006 & 611$\times$  & 110$\times$ \\
          EMT      & 10   & 577  & 3 & 0.225 \tiny$\pm$0.018 & 0.311 \tiny$\pm$0.025 & 0.0029 \tiny$\pm$0.0009 & 78$\times$   & 107$\times$ \\
          EB       & 50   & 2381 & 4 & 9.191 \tiny$\pm$0.505 & 0.384 \tiny$\pm$0.021 & 0.0031 \tiny$\pm$0.0013 & 2965$\times$ & 124$\times$ \\
          Cite     & 50   & 7476 & 3 & 0.107 \tiny$\pm$0.006 & 0.319 \tiny$\pm$0.008 & 0.0028 \tiny$\pm$0.0006 & 38$\times$   & 114$\times$ \\
          Mouse    & 50   & 4638 & 2 & 3.440 \tiny$\pm$0.194 & 0.198 \tiny$\pm$0.008 & 0.0021 \tiny$\pm$0.0007 & 1638$\times$ & 94$\times$  \\
          \bottomrule
        \end{tabular}
        }
      \end{sc}
    \end{small}
  \end{center}

  \vskip -0.1in
\end{table*}

\begin{figure}[t]
    \centering
        \includegraphics[width=1\linewidth]{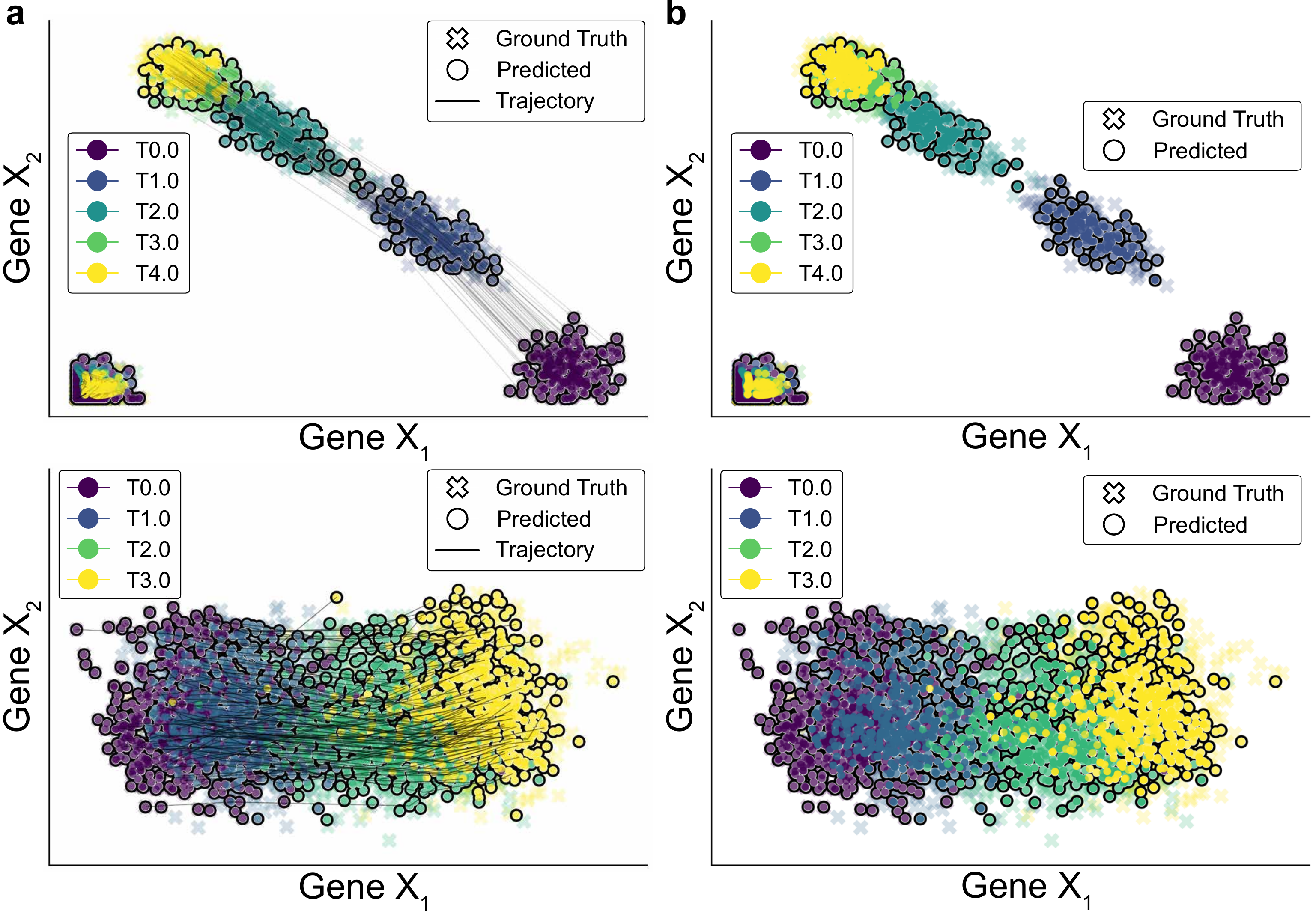}
\caption{
\textbf{Learned dynamics on the Gene and EMT datasets.}
Top: Gene dataset; bottom: EMT dataset.
(a) Individual cell trajectories with evolving weights.
(b) Population-level dynamics obtained by resampling cells according to normalized weights.
}  \label{fig:Gene}
\end{figure}

\paragraph{Accuracy preserved under fast inference (Q2).}
A natural concern is whether accuracy of distribution will decrease. We evaluate WFR-MFM on synthetic datasets (Table \ref{tab:Q21}) and real biological datasets (Table~\ref{tab:Q22} in Appendix~\ref{Appendix: held-out}) using $\mathcal{W}_1$ and RME. Across synthetic datasets (Gene, Dyngen, Gaussian), WFR-MFM consistently ranks first or second in both metrics. On real datasets (EMT, EB, CITE, Mouse), evaluations on held-out time points show a similar pattern, with WFR-MFM achieving the best score on CITE and top-2 performance on the remaining datasets.
Overall, these results demonstrate that WFR-MFM achieves fast inference without compromising accuracy. Figure~\ref{fig:Gene} visualizes Gene and EMT datasets from two complementary perspectives: (a) individual cell trajectories with a fixed cell count and evolving weights, and (b) changes in cell abundance obtained by normalizing weights into probabilities and resampling.

\begin{table*}[t]
  \caption{
Synthetic dataset performance under mean $\mathcal{W}_1$ and RME. RME is reported only for unbalanced methods. For stochastic inference methods, we report mean $\pm$ sd across five independent runs.
}
  
  \label{tab:Q21}

  \begin{center}
    \begin{small}
      \begin{sc}
        \resizebox{0.95\textwidth}{!}{
        \begin{tabular}{lcccccc}
          \toprule
          \textbf{Method} &
          \multicolumn{2}{c}{Gene (2D)} &
          \multicolumn{2}{c}{Dyngen (5D)} &
          \multicolumn{2}{c}{Gaussian (1000D)} \\
          \cmidrule(lr){2-3} \cmidrule(lr){4-5} \cmidrule(lr){6-7}
          & $\mathcal{W}_1$ ($\downarrow$) & RME ($\downarrow$)
          & $\mathcal{W}_1$ ($\downarrow$) & RME ($\downarrow$)
          & $\mathcal{W}_1$ ($\downarrow$) & RME ($\downarrow$) \\
          \midrule
          MMFM \citep{rohbeck2025modeling}              & 0.298                 & ---                  & 1.371                 & ---                  & 2.833                 & --- \\
          Metric FM \citep{kapusniak2024metric}        & 0.311                 & ---                  & 1.767                 & ---                  & 3.794                 & --- \\
          SF2M \citep{sflowmatch}                      & 0.224\tiny$\pm$0.007   & ---                  & 1.277\tiny$\pm$0.017   & ---                  & 3.543\tiny$\pm$0.002   & --- \\
          MIOFlow \citep{mioflow}                      & 0.148                 & ---                  & 0.965                 & ---                  & 2.858                 & --- \\
          TIGON \citep{TIGON}                          & 0.045                 & 0.014                & \underline{0.512}     & 0.047                & \underline{2.263}     & 0.127 \\
          DeepRUOT \citep{DeepRUOT}                    & 0.043\tiny$\pm$0.002   & 0.017\tiny$\pm$0.001 & 0.623\tiny$\pm$0.032   & 0.065\tiny$\pm$0.011 & 3.785\tiny$\pm$0.009   & 0.303\tiny$\pm$0.070 \\
          Var-RUOT \citep{sun2025variational}          & 0.079\tiny$\pm$0.003   & 0.008\tiny$\pm$0.002 & 0.522\tiny$\pm$0.008   & 0.177\tiny$\pm$0.007 & 2.813\tiny$\pm$0.004   & 0.041\tiny$\pm$0.006 \\
          UOT-FM \citep{eyring2024unbalancedness}      & 0.093                 & 0.010                & 1.204                 & 0.097                & 2.771                 & \underline{0.033} \\
          VGFM \citep{wang2025joint}                   & 0.046                 & 0.006                & 0.598                 & 0.037                & 3.010                 & 0.037 \\
          WFR-FM \citep{peng2026wfrfm}            & \textbf{0.019}         & \textbf{0.001}        & \textbf{0.135}         & \textbf{0.005}        & \textbf{2.233}         & 0.044 \\
          \textbf{WFR-MFM}                      & \underline{0.021}      & \underline{0.003}      & \underline{0.176}      & \underline{0.023}      & \textbf{2.233}         & \textbf{0.001}\\
          \bottomrule
        \end{tabular}
        }
      \end{sc}
    \end{small}
  \end{center}

  \vskip -0.1in
\end{table*}

\paragraph{Smooth speed--accuracy trade-off (Q3).}
We assess whether WFR-MFM enables a controllable balance between inference speed and accuracy by varying only the number of inference steps \(K\) in Algorithm~\ref{alg:WFR-MFM-inference}, without retraining the model. As $K$ increases, $\mathcal{W}_1$ and RME may exhibit minor fluctuations at small $K$ but decrease and stabilize overall (Fig.~\ref{fig:W1_RME_time_steps}a), while runtime grows approximately linearly (Fig.~\ref{fig:W1_RME_time_steps}b).
Together, these results demonstrate that WFR-MFM offers a continuous speed--accuracy trade-off: one-step inference maximizes efficiency, while a small number of steps achieves accuracy comparable to full dynamic WFR at a fraction of its computational cost.

\begin{figure}[t]
  \centering
  \includegraphics[width= 1\linewidth]{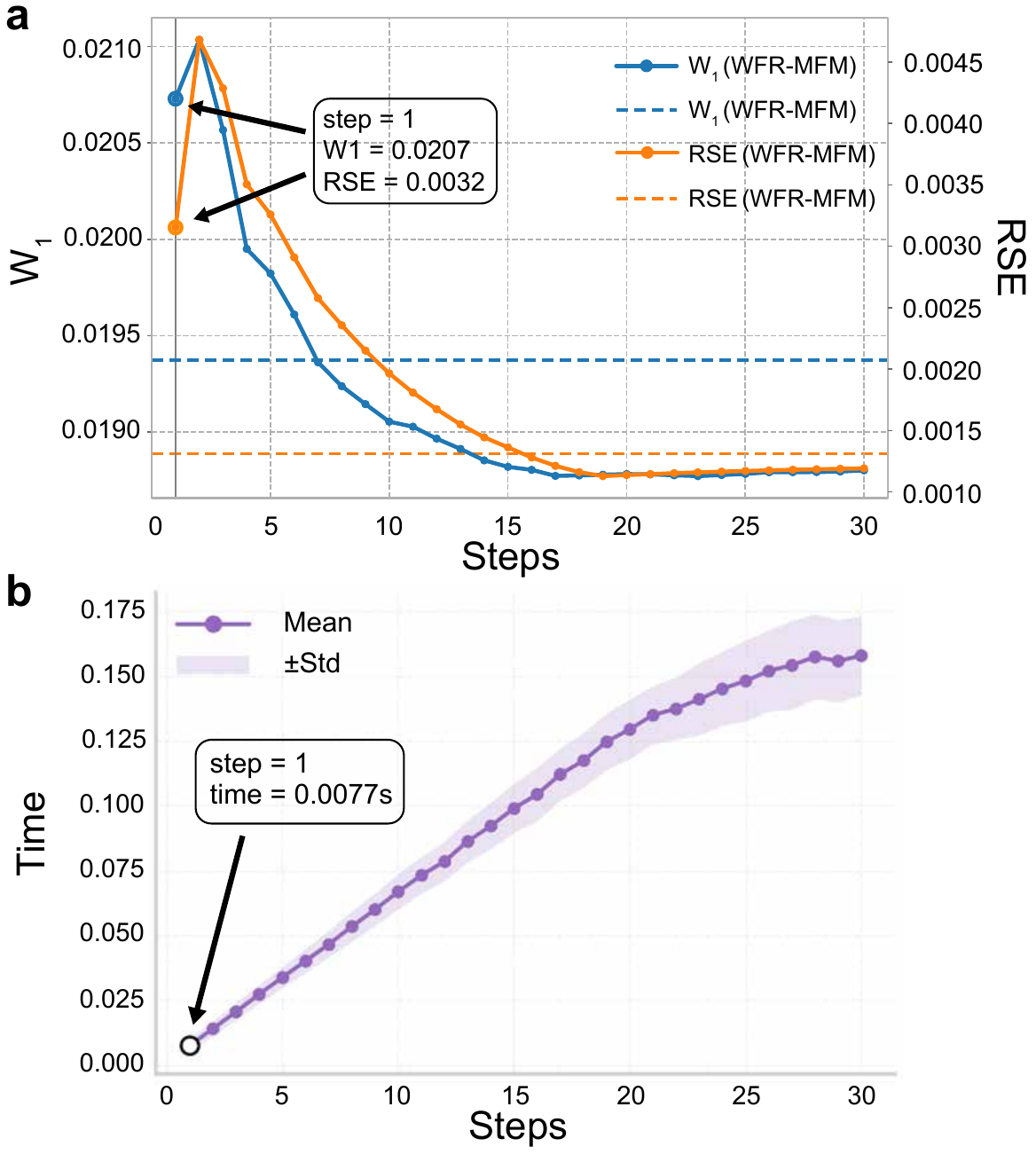}\\
  \caption{\small
  \textbf{Speed--accuracy trade-off (Q3) on the Simulation dataset.}
  (a) $\mathcal{W}_1$ and RME versus the number of inference steps.
  (b) Corresponding inference runtime versus the number of steps.
  For each step number, we perform 1000 independent runs.
  }
  \label{fig:W1_RME_time_steps}
\end{figure}

\paragraph{Scalability to Large-Scale Datasets (Q4).}
We evaluate WFR-MFM in a high-dimensional setting by comparing predictive accuracy and computational efficiency with multiple baseline methods,
and further assess its inference scalability and practical potential on a large-scale synthetic perturbation benchmark.
On the 100D EB dataset, we compare predictive accuracy ($\mathcal{W}_1$ distance) and computational efficiency
(training time, inference time, and memory usage) across multiple methods (Figure~\ref{fig:3d}). 
WFR-MFM achieves competitive prediction accuracy while significantly outperforming existing methods in inference speed, and also ranks among the top in terms of training cost and memory usage. Additional views are provided in Appendix~\ref{Appendix:EB}.
\begin{figure}[t]
    \centering
        \includegraphics[width=\linewidth]{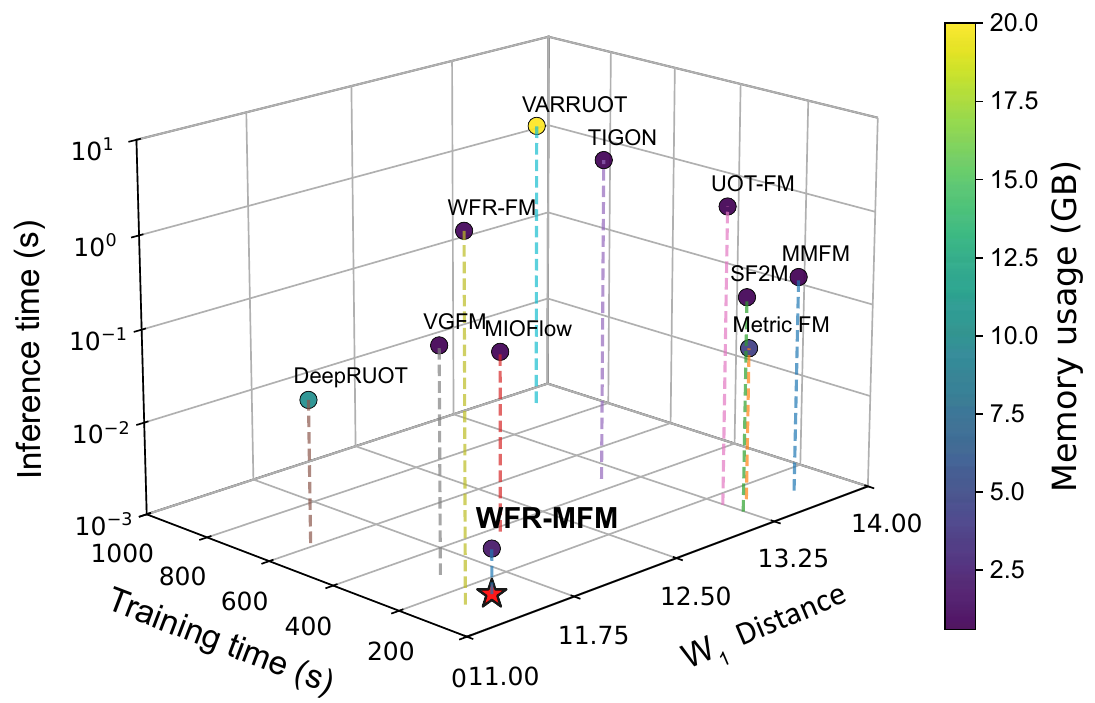}
\caption{
\textbf{Efficiency on the 100D EB dataset.}
Methods are compared by $\mathcal{W}_1$ distance, training time,
and inference time (log scale), with colors indicating GPU memory usage.
}
\label{fig:3d}
\end{figure}

In single-cell perturbation experiments, perturbation conditions grow combinatorially across modalities and experimental contexts \citep{klein2025cellflow}, making scalable and rapid inference essential.
To enable such large-scale perturbation studies within a single model, we condition the mean velocity field $\bm{v}$ and rate field $h$ on a perturbation embedding $\bm{c}$, i.e., $\bm{v}(\bm{x}, \bm{c}, t, T)$ and $h(\bm{x}, \bm{c}, t, T)$.
Details are provided in Appendix~\ref{Appendix: perturbation},
following strategies shown effective in prior work \citep{klein2025cellflow, rohbeck2025modeling}.
We evaluate large-scale perturbation inference on a synthetic benchmark with 5100 perturbation conditions (Appendix~\ref{sim perturb}), using 100 for training and 5000 for testing. 
WFR-MFM completes inference for all 5000 unseen conditions on 10,000 cells in 6.55 seconds,
whereas WFR-FM requires about 0.40 seconds for a representative perturbation condition,
implying roughly two thousand seconds for full-scale inference.
Despite this speedup, WFR-MFM maintains strong accuracy
($\mathcal{W}_1=0.115$, RME$=0.069$),
with predicted distributions closely matching ground truth
(Figure~\ref{fig:unseen_perturbations}).
These results show that WFR-MFM generalizes from limited observed perturbations to thousands of unseen conditions, enabling large-scale perturbation response prediction.
\begin{figure}[t]
\centering
\includegraphics[width=0.95\linewidth]{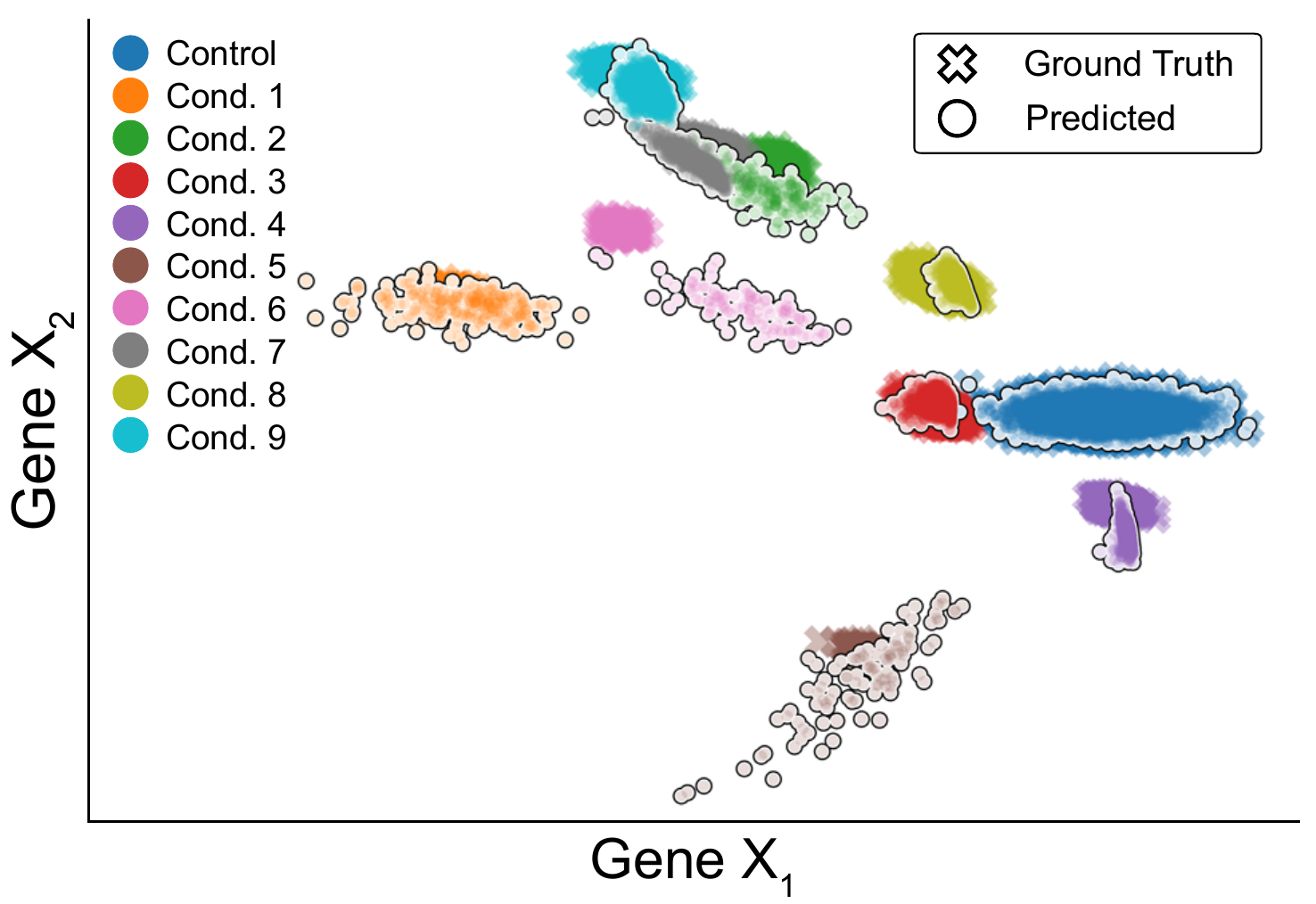}
\caption{
\textbf{Predictions on unseen perturbation conditions.}
Gene expression distributions before and after perturbation for nine unseen conditions,
visualized by resampling cells using predicted weights
(control denotes the unperturbed baseline).
}
\label{fig:unseen_perturbations}
\end{figure}

\section{Conclusion}
We formulate a mean-flow framework for unbalanced FM that summarizes transport and mass-growth dynamics over arbitrary time intervals via mean velocity and mass-growth fields, enabling fast one-step generation without trajectory simulation.
We then develop \textbf{WFR-MFM}, a mean flow matching algorithm for dynamic unbalanced OT under the WFR geometry with simulation-free inference.
Experiments on synthetic and real scRNA-seq datasets demonstrate orders-of-magnitude inference speedups while maintaining strong predictive accuracy, as well as a controllable speed--accuracy trade-off through multi-step consistency.

While WFR-MFM already achieves strong inference efficiency and accuracy, recent refinements to the mean-flow framework that improve stability and inference quality \citep{guo2025splitmeanflow,zhang2025alphaflow,geng2025improved} are not incorporated here and may further enhance its performance.
Extending WFR-MFM to large-scale experimental perturbation datasets and validating its practical utility in real-world settings remain important directions for future work. Beyond the WFR geometry, the proposed mean-flow perspective may further extend to other unbalanced transport formulations with well-defined Dirac-to-Dirac flows.

\section*{Acknowledgements}
This work was supported by the National Key R\&D Program of China (No.\ 2021YFA1003301 to T.L.) and the National Natural Science Foundation of China (NSFC Nos.\ 12288101 to T.L.\ \& P.Z., 8206100646 and T2321001 to P.Z.). We acknowledge support from the High-performance Computing Platform of Peking University for computation.

\bibliography{reference}
\bibliographystyle{unsrtnat}
\newpage
\appendix
\onecolumn
\section{Additional Theoretical and Method Details}\label{sec:appendix-methods}
\subsection{Details of WFR-MFM}\label{Appendix: WFR-MFM}
Most constructions in this subsection follow WFR-FM \citep{peng2026wfrfm}; the mean-flow training objective is the main departure.
\paragraph{WFR optimal transport.}
The dynamic formulation of unbalanced OT endowed with the WFR metric \citep{chizat2018interpolating,chizat2018unbalanced,liero2018optimal} is
\begin{equation}
\label{Dynamic WFR}
\begin{gathered}
\mathrm{WFR}_{\delta}^2(\mu_0,\mu_1)
=\inf_{\rho,g,\bm{u}} \int_0^1\!\int_{\mathcal{X}} \frac{1}{2}\big(\|\bm{u}(\bm{x},t)\|_2^2+\delta^2 \|g(\bm{x},t)\|_2^2\big)\,\rho_t(\bm{x})\,\mathrm{d}\bm{x}\,\mathrm{d}t \\
\text{s.t.}\qquad \partial_t\rho+\nabla_{\bm{x}}\!\cdot(\rho \bm{u})=\rho g,\quad \rho_0=\mu_0,\ \rho_1=\mu_1.
\end{gathered}
\end{equation}

The closed-form WFR distance between two Dirac measures $m_0\delta_{\bm{x}_0}$ and $m_1\delta_{\bm{x}_1}$ is given by \citep{chizat2018interpolating,liero2018optimal}:
\begin{equation*}
\label{dirac WFR}
\mathrm{WFR\text{-}DD}_{\delta}^2(m_0\delta_{\bm{x}_0},m_1\delta_{\bm{x}_1})
=2\delta^2\!\left(m_0+m_1-2\sqrt{m_0m_1}\,\overline{\operatorname{cos}}\big(\tfrac{\|\bm{x}_0-\bm{x}_1\|_2}{2\delta}\big)\right),
\end{equation*}
where $\overline{\operatorname{cos}}(\bm x)=\cos(\min\{ \bm x,\,\frac{\pi}{2}\})$. When $\|\bm{x}_0-\bm{x}_1\|_2<\pi\delta	$, the corresponding geodesic curve—i.e., the OT path, known as the traveling Dirac—is described by
\begin{equation}
\label{travelling dirac}
m(t)=At^2-2Bt+m_0,\qquad \bm{u}(\bm{x},t)\, m(t)=\bm{\omega}_0,
\end{equation}
where $A,B,\bm{\omega}_0$ are
\begin{equation}
\label{ABw}
\left\{
\begin{aligned}
&A=m_0+m_1-2\sqrt{\frac{m_0m_1}{1+\tau^2}},\qquad
B=m_0-\sqrt{\frac{m_0m_1}{1+\tau^2}},\\
&\bm{\omega}_0=2\delta\tau \sqrt{\frac{m_0m_1}{1+\tau^2}}\,\bm{l},\qquad
\tau=\operatorname{tan}\!\Big(\tfrac{\|\bm{x}_1-\bm{x}_0\|_2}{2\delta}\Big),
\end{aligned}
\right.
\end{equation}
and $\bm{l}$ is the unit vector pointing from $\bm{x}_0$ to $\bm{x}_1$.

\paragraph{Static semi-coupling form.}
An equivalent static (Kantorovich) formulation of the WFR problem \eqref{Dynamic WFR} is \citep{chizat2018unbalanced,liero2018optimal}
\begin{equation}\label{eq:wfr_static}
\mathrm{WFR}_\delta^2(\mu_0,\mu_1)
=\inf_{(\gamma_0,\gamma_1)\in(\mathcal{M}_+(\mathcal{X}^2))^2}
\int_{\mathcal{X}^2}\!
\mathrm{WFR\text{-}DD}_\delta^2\big(\gamma_0(\bm{x},\bm{y})\delta_{\bm{x}},\ \gamma_1(\bm{x},\bm{y})\delta_{\bm{y}}\big)\,\mathrm{d}\bm{x}\,\mathrm{d}\bm{y},
\end{equation}
subject to
\[
\Gamma(\mu_0,\mu_1)\stackrel{\mathrm{def.}}{=}\left\{(\gamma_0,\gamma_1)\in(\mathcal{M}_+(\mathcal{X}^2))^2:
\int_\mathcal{X}\gamma_0(\bm{x},\bm{y})\,\mathrm{d}\bm{y}=\mu_0(\bm{x}),\ 
\int_\mathcal{X}\gamma_1(\bm{x},\bm{y})\,\mathrm{d}\bm{x}=\mu_1(\bm{y})\right\}.
\]

\paragraph{Optimal Entropy Transport.}
Equivalently, WFR can be cast as an Optimal Entropy Transport (OET) problem \citep{chizat2018unbalanced,liero2018optimal}:
\begin{equation}\label{eq:oet}
\begin{aligned}
\mathrm{WFR}_\delta^2(\mu_0,\mu_1)
=2\delta^2\,\inf_{\gamma\in\mathcal{M}_+(\mathcal{X}^2)}\Big\{&
\int_{\mathcal{X}^2}\!\! \big[-2\ln\cos_+\!\tfrac{\|\bm{x}-\bm{y}\|_2}{2\delta}\big]\ \gamma(\bm{x},\bm{y})\,\mathrm{d}\bm{x}\,\mathrm{d}\bm{y} \\
&\quad + \mathrm{KL}\!\left(\int\gamma(\bm{x},\bm{y})\,\mathrm{d}\bm{y}\ \Big\|\ \mu_0(\bm{x})\right)
+ \mathrm{KL}\!\left(\int\gamma(\bm{x},\bm{y})\,\mathrm{d}\bm{x}\ \Big\|\ \mu_1(\bm{y})\right)\Big\}.
\end{aligned}
\end{equation}
When $\gamma$ is the optimal coupling of Eq.~\eqref{eq:oet}, the semi-coupling
\[
\gamma_0(\bm{x},\bm{y})=\frac{\gamma(\bm{x},\bm{y})}{\int_\mathcal{X}\gamma(\bm{x},\bm{z})\,\mathrm{d}\bm{z}}\,\mu_0(\bm{x}),\qquad
\gamma_1(\bm{x},\bm{y})=\frac{\gamma(\bm{x},\bm{y})}{\int_\mathcal{X}\gamma(\bm{z},\bm{y})\,\mathrm{d}\bm{z}}\,\mu_1(\bm{y})
\]
solves the static WFR problem in Eq.~\eqref{eq:wfr_static}.
\paragraph{Conditional measure path.}
Given $\bm{z}\sim q(\bm{z})$, we consider the decoupled conditional path
\[
\rho_t(\bm{x}\mid \bm{z}) \;=\; m_t(\bm{z})\,\tilde\rho_t(\bm{x}\mid \bm{z}),
\]
with conditional dynamics
\begin{equation}\label{eq:cond_cont}
\partial_t\rho_t(\bm{x}\mid \bm{z})+\nabla_{\bm{x}}\!\cdot\!\big(\rho_t(\bm{x}\mid \bm{z})\,\bm{u}_t(\bm{x}\mid \bm{z})\big)
= g_t(\bm{x}\mid \bm{z})\,\rho_t(\bm{x}\mid \bm{z}).
\end{equation}

\paragraph{Marginalization.}
Define the marginals by
\begin{equation}\label{eq:marginal_defs}
\rho_t(\bm{x})=\int \rho_t(\bm{x}\mid \bm{z})\,q(\bm{z})\,\mathrm{d}\bm{z},\quad
\bm{u}_t(\bm{x})=\int \bm{u}_t(\bm{x}\mid \bm{z})\,\tfrac{\rho_t(\bm{x}\mid \bm{z})q(\bm{z})}{\rho_t(\bm{x})}\,\mathrm{d}\bm{z},\quad
g_t(\bm{x})=\int g_t(\bm{x}\mid \bm{z})\,\tfrac{\rho_t(\bm{x}\mid \bm{z})q(\bm{z})}{\rho_t(\bm{x})}\,\mathrm{d}\bm{z}.
\end{equation}
If $q(\bm{z})$ is independent of $(\bm{x},t)$ and Eq.~\eqref{eq:cond_cont} holds for each $\bm{z}$, then Eq.~\eqref{eq:marginal_defs} satisfies the continuity equation with source term
\begin{equation}\label{eq:cont_src}
\partial_t\rho_t(\bm{x})+\nabla_{\bm{x}}\!\cdot(\rho_t(\bm{x})\bm{u}_t(\bm{x}))=g_t(\bm{x})\,\rho_t(\bm{x}).
\end{equation}

\paragraph{Conditional Gaussian measure path.}
We take the conditional Gaussian measure path (CGMP)
\begin{equation}\label{eq:cgmp}
\tilde\rho_t(\bm{x}\mid \bm{z})=\mathcal{N}\!\big(\bm{x}\,\big|\,\bm{\eta}_t(\bm{z}),\,\bm{\sigma}_t^2(\bm{z})\big),\qquad
\rho_t(\bm{x}\mid \bm{z})=m_t(\bm{z})\,\tilde\rho_t(\bm{x}\mid \bm{z}),
\end{equation}
with explicit conditional field and rate
\begin{equation}\label{eq:cgmp_field}
\bm{u}_t(\bm{x}\mid \bm{z})=\frac{\bm{\sigma}_t'(\bm{z})}{\bm{\sigma}_t(\bm{z})}\big(\bm{x}-\bm{\eta}_t(\bm{z})\big)+\bm{\eta}_t'(\bm{z}),
\qquad
g_t(\bm{x}\mid \bm{z})=\partial_t\ln m_t(\bm{z}).
\end{equation}

\paragraph{WFR coupling solves the dynamic WFR problem.}
Consider two measures $\mu_0(\bm{x})$ and $\mu_1(\bm{x})$ as the source and target of a WFR problem, and denote by $(\gamma_0,\gamma_1)$ a semi-coupling solving the static form. Without loss of generality, let $\mu_0(\bm{x})$ be a probability density on $\mathcal{X}$, so that $\gamma_0(\bm{x}_0,\bm{x}_1)$ is a probability density on $\mathcal{X}^2$. Let $\bm{z}=(\bm{x}_0,\bm{x}_1)\sim\gamma_0(\bm{x}_0,\bm{x}_1)$ represent the source and target locations. 
Motivated by the traveling Dirac solution of the dynamic WFR problem \citep{chizat2018interpolating} and following the construction in WFR-FM \citep{peng2026wfrfm}, we adopt a special CGMP, termed the traveling Gaussian.
\begin{equation}
\label{travelling Gaussian}
\begin{aligned}
&\tilde{\rho}_t(\bm{x}\mid \bm{x}_0,\bm{x}_1)
= \mathcal{N}\!\Big(\bm{x}\,\Big|\, \bm{x}_0+\bm{\omega}_0\!\int_0^t\frac{\mathrm{d}s}{m_s(\bm{x}_0,\bm{x}_1)},\,\sigma^2\mathbf{I}\Big)
:= \mathcal{N}\!\big(\bm{x}\,\big|\,\bm{x}_0+\bm{\omega}_0\Lambda_t(\bm{x}_0,\bm{x}_1),\,\sigma^2\mathbf{I}\big)\\
&= \mathcal{N}\!\Big(\bm{x}\,\Big|\, \bm{x}_0+\frac{\bm{\omega}_0}{\sqrt{m_0 A - B^2}}
\Big( \arctan \!\frac{At - B}{\sqrt{m_0 A - B^2}} - \arctan \!\frac{-B}{\sqrt{m_0 A - B^2}} \Big),\,\sigma^2\mathbf{I}\Big),
\end{aligned}
\end{equation}
where $\bm{\omega}_0$ and $m_t(\bm{x}_0,\bm{x}_1)$ are defined by Eqs.~\eqref{travelling dirac} and \eqref{ABw}, with boundary condition $m_0(\bm{x}_0,\bm{x}_1)=1$ and $m_1(\bm{x}_0,\bm{x}_1)=\frac{\gamma_1(\bm{x}_0,\bm{x}_1)}{\gamma_0(\bm{x}_0,\bm{x}_1)}$. Under this construction, the marginal boundary measures converge to $\mu_0$ and $\mu_1$ respectively as $\sigma\to 0$, and the induced measure path follows the WFR geodesic.

\paragraph{Training of WFR-MFM.}
Under the WFR geometry, the mean-flow framework yields the following training losses for WFR-MFM.
The mean velocity and mass-growth fields $(\bm v, h)$ are parameterized by neural networks
$(\bm v_{\bm{\theta}}, h_{\bm{\phi}})$. The unbalanced mean flow matching objective is given by
\begin{equation*}
\begin{aligned}
\mathcal{L}(\bm{\theta},\bm{\phi})
=\  \mathbb{E}_{t<T,\;\bm{x}\sim\rho_t(\bm{x})}
\Big[
\| \bm{v}_{\bm{\theta}}(\bm{x},t,T) - \text{sg}(\bm{v}(\bm{x},t,T))\|_2^2
+ \lambda\,\| h_{\bm{\phi}}(\bm{x},t,T) - \text{sg}(h(\bm{x},t,T))\|_2^2
\Big],		
\end{aligned}
\end{equation*}
where $\lambda>0$ balances the supervision of the velocity and mass-growth fields, and $\text{sg}(\cdot)$ is the stop-gradient operator. 
The regression targets $(\bm v, h)$ are computed from the mean-field derivative identities:
\begin{equation}
\begin{gathered}
\bm{v}(\bm{x}, t, T )
=  \bm{u}_t(\bm{x} )
+ (T - t)\!\Bigr[
\partial_t \bm{v}_{\bm{\theta}}(\bm{x},t,T)
+ (\nabla_{\bm{x}} \bm{v}_{\bm{\theta}}(\bm{x},t,T))\,\bm{u}_t(\bm{x} )
\Bigr],\\
h(\bm{x}, t, T )
= g_t(\bm{x} )
+ (T - t)\!\Bigr[
\partial_t h_{\bm{\phi}}(\bm{x},t,T)
+ \nabla_{\bm{x}} h_{\bm{\phi}}(\bm{x},t,T)\cdot \bm{u}_t(\bm{x})
\Bigr].
\end{gathered}	
\end{equation}
For conditional training, the corresponding objective takes the form
\begin{equation*}
\begin{aligned}
\mathcal{L}_{\mathrm{c}}(\bm{\theta},\bm{\phi}) =\ &  \mathbb{E}_{t<T,\bm{z}\sim q(\bm{z}), \bm{x}\sim \tilde{\rho}_t(\bm{x}\vert \bm{z})}
\Big[
\| \bm{v}_{\bm{\theta}}(\bm{x},t,T) 
- \text{sg}(\bm{v}(\bm{x},t,T\mid\bm{z}))\|_2^2+ \lambda\,\| h_{\bm{\phi}}(\bm{x},t,T) 
- \text{sg}(h(\bm{x},t,T\mid\bm{z}))\|_2^2
\Big]\, m_{t}(\bm{z}),		
\end{aligned}
\end{equation*}
where the conditional regression targets are defined as
\begin{equation*}
\begin{aligned}
\bm{v}(\bm{x}, t, T \mid \bm{z})
=\ & \bm{u}_t(\bm{x} \mid \bm{z})
+ (T - t)\,\Bigl[
\partial_t \bm{v}_{\bm{\theta}}(\bm{x},t,T) 
+ (\nabla_{\bm{x}} \bm{v}_{\bm{\theta}}(\bm{x},t,T))\,\bm{u}_t(\bm{x} \mid \bm{z})
\Bigr],\\[0.4em]
h(\bm{x}, t, T \mid \bm{z})
= \ & g_t(\bm{x} \mid \bm{z})
+ (T - t)\,\Bigl[
\partial_t h_{\bm{\phi}}(\bm{x},t,T) 
+ \nabla_{\bm{x}} h_{\bm{\phi}}(\bm{x},t,T)\cdot \bm{u}_t(\bm{x} \mid \bm{z})
\Bigr].
\end{aligned}
\end{equation*}
The conditional targets $\bm{v}(\bm{x},t,T\mid\bm{z})$ and
$h(\bm{x},t,T\mid\bm{z})$ are defined analogously using the conditional instantaneous fields
$\bm{u}_t(\bm{x}\mid\bm{z})$ and $g_t(\bm{x}\mid\bm{z})$.

\subsection{Conditional Extension of WFR-MFM for Perturbations}
\label{Appendix: perturbation}

Many biological datasets involve multiple perturbation conditions, where a common control distribution $\mu_0$ is mapped to a family of perturbed endpoints $\{\mu_1^{\bm{c}}\}_{\bm{c}\in\mathcal{\bm{C}}}$, with $\bm{c}$ indexing the perturbation identity.  
Under the WFR formula, this perturbation induces condition-dependent transport and mass-growth dynamics, resulting in conditional fields $\bm{u}_t(\bm{x},\bm{c})$ and $g_t(\bm{x},\bm{c})$.

We extend the mean-flow formulation to this perturbation setting by conditioning the mean velocity and mean mass-growth fields on a perturbation embedding $\bm{c}$, defined as
\[
\bm{v}(\bm{x},\bm{c}, t, T)
\;=\;
\frac{1}{T - t}\int_t^T \bm{u}_\tau(\bm{x}_\tau,\bm{c})\,\mathrm{d}\tau,
\;\;\;
h(\bm{x},\bm{c}, t, T)
\;=\;
\frac{1}{T - t}\int_t^T g_\tau(\bm{x}_\tau,\bm{c})\,\mathrm{d}\tau.
\]

With this extension, we parameterize the conditional mean fields by neural networks
$\bm{v}_{\bm{\theta}}(\bm{x},\bm{c}, t, T)$ and $h_{\bm{\phi}}(\bm{x},\bm{c}, t, T)$,
and define the conditional training objective for learning perturbation-dependent mean fields as
\begin{equation*}
	\begin{aligned}
\mathcal{L}_{\mathrm{c}}(\bm{\theta},\bm{\phi})
=\mathbb{E}_{ t<T,\, \bm{c},\, \bm{z},\, \bm{x}}
\Big[
\| \bm{v}_{\bm{\theta}}(\bm{x},\bm{c},t,T) - \text{sg}(\bm{v}(\bm{x},\bm{c},t,T\mid\bm{z}))\|_2^2
+ \lambda\,\| h_{\bm{\phi}}(\bm{x},\bm{c},t,T) - \text{sg}(h(\bm{x},\bm{c},t,T\mid\bm{z}))\|_2^2
\Big]\, m_{t}(\bm{z}),
\end{aligned}
\end{equation*}
where 
\begin{equation}
\label{eq:conditional-mean-fields}
\begin{gathered}
\bm{v}(\bm{x},\bm{c},t,T\mid\bm{z})
= \bm{u}_t(\bm{x},\bm{c}\mid\bm{z})
+ (T - t)\!\left[
    \partial_t \bm{v}_{\bm{\theta}}(\bm{x},\bm{c},t,T)
    + (\nabla_{\bm{x}}\bm{v}_{\bm{\theta}}(\bm{x},\bm{c},t,T))\,\bm{u}_t(\bm{x},\bm{c}\mid\bm{z})
\right], \\[0.4em]
h(\bm{x},\bm{c},t,T\mid\bm{z})
= g_t(\bm{x},\bm{c}\mid\bm{z})
+ (T - t)\!\left[
    \partial_t h_{\bm{\phi}}(\bm{x},\bm{c},t,T)
    + \nabla_{\bm{x}} h_{\bm{\phi}}(\bm{x},\bm{c},t,T)\cdot \bm{u}_t(\bm{x},\bm{c}\mid\bm{z})
\right].
\end{gathered}
\end{equation}

Intuitively, after conditioning on the perturbation embedding $\bm{c}$, both one-step and multi-step inference follow the same update rule as in the base model; only the underlying mean fields are altered in a perturbation-specific manner. This enables WFR-MFM to support large perturbation collections within a single unified model, instead of training separate models for different perturbations. The complete training procedure is summarized in Algorithm~\ref{alg:WFR-MFM-perturb}.

\begin{algorithm}[tb]
\caption{Conditional WFR-MFM Training for Perturbations}
\label{alg:WFR-MFM-perturb}
\begin{algorithmic}
 \STATE {\bfseries Input:} Perturbation set $\mathcal C$; Control distribution $\mu_0$; perturbed distributions $\{\mu_1^c\}_{c\in\mathcal C}$; bandwidth $\sigma$; OET batch size $B$; training batch size $b$; WFR penalty $\{\delta^c\}_{c\in\mathcal C}$; condition embeddings $\{\bm{c}\}_{c\in\mathcal C}$; condition batch size $B_c$; mean fields $\bm{v}_{\bm{\theta}}(\bm{x},\bm{c},t,T)$ and $h_{\bm{\phi}}(\bm{x},\bm{c},t,T)$
\STATE {\bfseries Output:} trained $\bm{v}_{\bm{\theta}}$ and $h_{\bm{\phi}}$

\STATE {\bfseries Precompute couplings}
\FOR{$c \in \mathcal{C}$}
  \STATE $\gamma^{(c)} \gets \mathrm{OET}(\mu_0,\,\mu_1^{(c)})$
  \STATE $\displaystyle
  \gamma_0^{(c)}(\bm{x},\bm{y}) \gets
  \frac{\gamma^{(c)}(\bm{x},\bm{y})}{\int \gamma^{(c)}(\bm{x},\bm{z})\,\mathrm{d}\bm{z}}\,\mu_0(\bm{x})$
  \STATE Store $\gamma_0^{(c)}$ in memory
\ENDFOR

\WHILE{Training}

  \STATE Sample perturbation minibatch $\mathcal S \subset \mathcal C$ with $|\mathcal S| = B_c$
  \FOR{$c \in \mathcal S$}
    \STATE Retrieve $\gamma_0^{(c)}$ and embedding $\bm{c}$
    \STATE Sample $b$ pairs $(\bm{x}_0,\bm{x}_1) \sim \gamma_0^{(c)}$
    \STATE Compute travelling-WFR quantities $(A^{(c)}, B^{(c)}, \bm{\omega}_0^{(c)}, \tau^{(c)}, m_t^{(c)})$
    \STATE Sample $(t,T)$
    \STATE $\bm{\eta}_t \gets \bm{x}_0 + \bm{\omega}_0^{(c)}\,\Lambda_t(\bm{x}_0,\bm{x}_1)$
    \STATE Sample $\bm{x} \sim \mathcal{N}(\bm{\eta}_t,\sigma^2\mathbf I)$
    \STATE $\bm{u} \gets \bm{\omega}_0^{(c)} / m_t^{(c)}(\bm{x}_0,\bm{x}_1)$
    \STATE $g \gets \tfrac{d}{dt}\ln m_t^{(c)}(\bm{x}_0,\bm{x}_1)$
    \STATE $m \gets m_t^{(c)}(\bm{x}_0,\bm{x}_1) / \gamma_0^{(c)}(\bm{x}_0,\bm{x}_1)$
    \STATE $\bm{v}^{(c)} \gets \bm{u} +\text{sg}\!\left( (T{-}t)\left[
      \partial_t \bm{v}_{\bm{\theta}}(\bm{x},\bm{c},t,T)
      + (\nabla_{\bm{x}}\bm{v}_{\bm{\theta}}(\bm{x},\bm{c},t,T))\bm{u}
    \right]\right)$
    \STATE $h^{(c)} \gets g + \text{sg}\!\left( (T{-}t)\left[
      \partial_t h_{\bm{\phi}}(\bm{x},\bm{c},t,T)
      + \nabla_{\bm{x}} h_{\bm{\phi}}(\bm{x},\bm{c},t,T)\!\cdot\!\bm{u}
    \right]\right)$
    \STATE Collate samples into batched tensors $\{\bm{x}^b, t^b, T^b, \bm{v}^b, h^b, m^b\}$.
    \STATE $\mathcal{L}_c(\bm{\theta},\bm{\phi}) \gets
      \Big(\|\bm{v}_{\bm{\theta}}(\bm{x}^b,\bm{c},t^b,T^b) - \bm{v}^b\|_2^2
      + \lambda\,\|h_{\bm{\phi}}(\bm{x}^b,\bm{c},t^b,T^b) - h^{b}\|_2^2\Big)\, m^b$
  \ENDFOR
  \STATE $\mathcal{L} \gets \frac{1}{B_c}\sum_{c\in\mathcal S}\mathcal{L}_c$
  \STATE Update $(\bm{\theta},\bm{\phi})$ using $\nabla \mathcal{L}$
\ENDWHILE
\STATE \textbf{return} $\bm{v}_{\bm{\theta}},\, h_{\bm{\phi}}$
\end{algorithmic}
\end{algorithm}

\section{Proofs}
\label{sec:appendix-proofs}
\subsection{Proof of Theorem \ref{thm1}}\label{proof: thm1}
\textbf{Theorem \ref{thm1}.}
\textit{If $\rho_t(\bm{x}) > 0$ for all $\bm{x} \in \mathcal{X}$ and $t \in [0,1],$ and $q(\bm{z})$ is independent of $(\bm{x},t,T)$, then
\begin{equation}\label{eq:appendix thm1}
\begin{aligned}
\mathcal{L}(\bm{\theta},\bm{\phi})
= \mathcal{L}_{\mathrm{c}}(\bm{\theta},\bm{\phi}) + C,
\end{aligned}
\end{equation}
for some constant $C$ independent of $(\bm{\theta},\bm{\phi})$. Consequently,
\[
\nabla_{\bm{\theta},\bm{\phi}}\mathcal{L}(\bm{\theta},\bm{\phi})
= \nabla_{\bm{\theta},\bm{\phi}}\mathcal{L}_{\mathrm{c}}(\bm{\theta},\bm{\phi}).
\]}

\begin{proof}
We first restate the unconditional and conditional mean flow matching losses.
The \textbf{unconditional mean flow matching loss} is defined as
\[
\mathcal{L}(\bm{\theta},\bm{\phi})
= \mathbb{E}_{t<T,\;\bm{x}\sim\rho_t(\bm{x})}
\Big[
\| \bm{v}_{\bm{\theta}}(\bm{x},t,T) - \text{sg}(\bm{v}(\bm{x},t,T))\|_2^2
+ \lambda\,\| h_{\bm{\phi}}(\bm{x},t,T) - \text{sg}(h(\bm{x},t,T))\|_2^2
\Big],
\]
where
\begin{equation}\label{aeq1}
\begin{aligned}
\bm{v}(\bm{x}, t, T )
&= \bm{u}_t(\bm{x})
+ (T - t)\!\left[
\partial_t \bm{v}_{\bm{\theta}}(\bm{x},t,T)
+ (\nabla_{\bm{x}} \bm{v}_{\bm{\theta}}(\bm{x},t,T))\,\bm{u}_t(\bm{x})
\right],\\[0.3em]
h(\bm{x}, t, T )
&= g_t(\bm{x})
+ (T - t)\!\left[
\partial_t h_{\bm{\phi}}(\bm{x},t,T)
+ \nabla_{\bm{x}} h_{\bm{\phi}}(\bm{x},t,T)\cdot\bm{u}_t(\bm{x})
\right].
\end{aligned}
\end{equation}

The corresponding \textbf{conditional mean flow matching loss} is
\[
\mathcal{L}_{\mathrm{c}}(\bm{\theta},\bm{\phi})
= \mathbb{E}_{t<T,\bm{z}\sim q(\bm{z}), \bm{x}\sim\rho_t(\bm{x})}
\Big[
\| \bm{v}_{\bm{\theta}}(\bm{x},t,T) - \text{sg}(\bm{v}(\bm{x},t,T\mid\bm{z}))\|_2^2
+ \lambda\,\| h_{\bm{\phi}}(\bm{x},t,T) - \text{sg}(h(\bm{x},t,T\mid\bm{z}))\|_2^2
\Big]\, m_{t}(\bm{z}),
\]
where $\text{sg}(\cdot)$ denotes the stop-gradient operator, preventing gradient flow through the target functions. The conditional targets are defined via derivative identities:
\[
\begin{aligned}
\bm{v}(\bm{x}, t, T \mid \bm{z})
&= \bm{u}_t(\bm{x} \mid \bm{z})
+ (T - t)\!\left[
\partial_t \bm{v}_{\bm{\theta}}(\bm{x},t,T)
+ (\nabla_{\bm{x}} \bm{v}_{\bm{\theta}}(\bm{x},t,T))\,\bm{u}_t(\bm{x} \mid \bm{z})
\right],\\[0.4em]
h(\bm{x}, t, T \mid \bm{z})
&= g_t(\bm{x} \mid \bm{z})
+ (T - t)\!\left[
\partial_t h_{\bm{\phi}}(\bm{x},t,T)
+ \nabla_{\bm{x}} h_{\bm{\phi}}(\bm{x},t,T)\cdot \bm{u}_t(\bm{x} \mid \bm{z})
\right].
\end{aligned}
\]
We prove Eq.~\eqref{eq:appendix thm1} in 5 steps.

\paragraph{Step 1. Expansion of the unconditional $\bm{v}$-related loss.} 
The component of $\mathcal{L}(\bm{\theta},\bm{\phi})$ associated with $\bm{v}$ can be expanded as
\begin{equation}\label{aeq2}
\begin{aligned}
&\mathbb{E}_{t<T,\;\bm{x}\sim\rho_t(\bm{x})}
\big\| \bm{v}_{\bm{\theta}}(\bm{x},t,T) - \text{sg}(\bm{v}(\bm{x},t,T))\big\|_2^2 \\[0.3em]
&= \mathbb{E}_{t<T,\;\bm{x}\sim\rho_t(\bm{x})}
\Big[
\|\bm{v}_{\bm{\theta}}(\bm{x},t,T)\|_2^2
- 2\langle \bm{v}_{\bm{\theta}}(\bm{x},t,T),\, \text{sg}(\bm{v}(\bm{x},t,T)) \rangle
+ \|\text{sg}(\bm{v}(\bm{x},t,T))\|_2^2
\Big] \\[0.3em]
&= \mathbb{E}_{t<T,\;\bm{x}\sim\rho_t(\bm{x})}\!
\|\bm{v}_{\bm{\theta}}(\bm{x},t,T)\|_2^2
- 2\,\mathbb{E}_{t<T,\;\bm{x}\sim\rho_t(\bm{x})}\!
\langle \bm{v}_{\bm{\theta}}(\bm{x},t,T),\, \text{sg}(\bm{v}(\bm{x},t,T)) \rangle
+ \tilde{C}_1,
\end{aligned}
\end{equation}
where $\tilde{C}_1 = \mathbb{E}_{t<T,\;\bm{x}\sim\rho_t(\bm{x})}\|\text{sg}(\bm{v}(\bm{x},t,T))\|_2^2$ is constant with respect to $\bm{\theta}$.

\paragraph{Step 2. Transformation of the cross term.}
We next analyze the cross term in~Eq.~\eqref{aeq2}:
\begin{equation}\label{aeq3}
\begin{aligned}
&\mathbb{E}_{t<T,\;\bm{x}\sim\rho_t(\bm{x})}
\!\big\langle \bm{v}_{\bm{\theta}}(\bm{x},t,T),\, \text{sg}(\bm{v}(\bm{x},t,T)) \big\rangle \\[0.3em]
&\overset{(1)}{=}
\mathbb{E}_{t<T,\;\bm{x}\sim\rho_t(\bm{x})}
\Big\langle \bm{v}_{\bm{\theta}}(\bm{x},t,T),\,
\text{sg}\!\left(
\bm{u}_t(\bm{x})
+ (T-t)\!\left[
\partial_t \bm{v}_{\bm{\theta}}(\bm{x},t,T)
+ (\nabla_{\bm{x}} \bm{v}_{\bm{\theta}}(\bm{x},t,T))\,\bm{u}_t(\bm{x})
\right]
\right)
\Big\rangle \\[0.4em]
&= \mathbb{E}_{t<T}\!\int_{\mathcal{X}}
\Big\langle \bm{v}_{\bm{\theta}}(\bm{x},t,T),\,
\text{sg}\!\left(
\bm{u}_t(\bm{x})\rho_t(\bm{x})
+ (T-t)\!\left[
\partial_t \bm{v}_{\bm{\theta}}(\bm{x},t,T)\rho_t(\bm{x})
+ (\nabla_{\bm{x}} \bm{v}_{\bm{\theta}}(\bm{x},t,T))\,\bm{u}_t(\bm{x})\rho_t(\bm{x})
\right]
\right)
\Big\rangle d\bm{x} \\[0.4em]
&\overset{(2)}{=}
\mathbb{E}_{t<T}\!\int_{\mathcal{X}}\!\int_{\mathcal{Z}}
\Big\langle \bm{v}_{\bm{\theta}}(\bm{x},t,T),\,
\text{sg}\!\Big(
\bm{u}_t(\bm{x}\!\mid\! \bm{z})
+ (T-t)\!\big[
\partial_t \bm{v}_{\bm{\theta}}(\bm{x},t,T)
+ (\nabla_{\bm{x}} \bm{v}_{\bm{\theta}}(\bm{x},t,T))\,\bm{u}_t(\bm{x}\!\mid\! \bm{z})
\big]
\Big)\!\Big\rangle
\rho_t(\bm{x}\!\mid\! \bm{z})q(\bm{z})\,d\bm{z}\,d\bm{x} \\[0.4em]
&= \mathbb{E}_{t<T,\;\bm{z}\sim q(\bm{z}),\;\bm{x}\sim\rho_t(\bm{x}\mid \bm{z})}
\!\big\langle \bm{v}_{\bm{\theta}}(\bm{x},t,T),\, \text{sg}(\bm{v}(\bm{x},t,T\mid \bm{z})) \big\rangle,
\end{aligned}
\end{equation}
where  
(1) applies the definition of $\bm{v}(\bm{x},t,T)$ from~Eq.~\eqref{aeq1}, and  
(2) uses the marginal flow velocity
\[
\bm{u}_t(\bm{x})
= \frac{\displaystyle\int \bm{u}_t(\bm{x}\mid \bm{z})\,\rho_t(\bm{x}\mid \bm{z})\,q(\bm{z})\,d\bm{z}}
       {\rho_t(\bm{x})}
\]
and
\begin{equation*}
	\begin{gathered}
\rho_t(\bm{x})=\int\rho_t(\bm{x}\vert \bm{z})q(\bm{z})\mathrm{d} \bm{z}.	
	\end{gathered}
\end{equation*}

\paragraph{Step 3. Conversion to the conditional expectation.}
Substituting~Eq.~\eqref{aeq3} into~Eq.~\eqref{aeq2}, we obtain
\begin{equation}\label{aeq4}
\begin{aligned}
&\mathbb{E}_{t<T,\;\bm{x}\sim\rho_t(\bm{x})}
\big\| \bm{v}_{\bm{\theta}}(\bm{x},t,T) - \text{sg}(\bm{v}(\bm{x},t,T))\big\|_2^2 \\[0.3em]
&=
\mathbb{E}_{t<T,\;\bm{z}\sim q(\bm{z}),\;\bm{x}\sim\rho_t(\bm{x}\mid \bm{z})}
\|\bm{v}_{\bm{\theta}}(\bm{x},t,T)\|_2^2
- 2\,\mathbb{E}_{t<T,\;\bm{z}\sim q(\bm{z}),\;\bm{x}\sim\rho_t(\bm{x}\mid \bm{z})}
\!\big\langle \bm{v}_{\bm{\theta}}(\bm{x},t,T),\, \text{sg}(\bm{v}(\bm{x},t,T\mid \bm{z})) \big\rangle
+ \tilde{C}_1 \\[0.5em]
&= \mathbb{E}_{t<T,\;\bm{z}\sim q(\bm{z}),\;\bm{x}\sim\rho_t(\bm{x}\mid \bm{z})}
\Big[
\|\bm{v}_{\bm{\theta}}(\bm{x},t,T)\|_2^2
- 2\big\langle \bm{v}_{\bm{\theta}}(\bm{x},t,T),\, \text{sg}(\bm{v}(\bm{x},t,T\mid \bm{z})) \big\rangle
\Big]
+ \tilde{C}_1 \\[0.5em]
&= \mathbb{E}_{t<T,\;\bm{z}\sim q(\bm{z}),\;\bm{x}\sim \rho_t(\bm{x}\mid \bm{z})}\;
\!\big\| \bm{v}_{\bm{\theta}}(\bm{x},t,T) - \text{sg}(\bm{v}(\bm{x},t,T\mid \bm{z}))\big\|_2^2
 + C_1,
\end{aligned}
\end{equation}
where $C_1$ is a constant independent of $(\bm{\theta},\bm{\phi})$.

\paragraph{Step 4. Extension to the mass-growth rate field.}
A similar derivation holds for the $h$-term:
\begin{equation}\label{aeq5}
\mathbb{E}_{t<T,\;\bm{x}\sim\rho_t(\bm{x})}
\lambda\,\| h_{\bm{\phi}}(\bm{x},t,T) - \text{sg}(h(\bm{x},t,T))\|_2^2
= \mathbb{E}_{t<T,\;\bm{z}\sim q(\bm{z}),\;\bm{x}\sim\rho_t(\bm{x}\mid \bm{z})}
\lambda\,\| h_{\bm{\phi}}(\bm{x},t,T) - \text{sg}(h(\bm{x},t,T\mid \bm{z}))\|_2^2\, 
+ C_2.	
\end{equation}

\paragraph{Step 5. Conclusion.}
Combining Eqs.~\eqref{aeq4} and \eqref{aeq5}, we conclude that
\[
\mathcal{L}(\bm{\theta},\bm{\phi})
= \mathcal{L}_{\mathrm{c}}(\bm{\theta},\bm{\phi}) + C,
\]
for some constant $C$ independent of $(\bm{\theta},\bm{\phi})$.  
Hence,
\[
\nabla_{\bm{\theta},\bm{\phi}}\mathcal{L}(\bm{\theta},\bm{\phi})
= \nabla_{\bm{\theta},\bm{\phi}}\mathcal{L}_{\mathrm{c}}(\bm{\theta},\bm{\phi}),
\]
which completes the proof.
\end{proof}

\subsection{Failure of Time-Averaged Conditional Mean Fields}
\label{proof: thm2}
This appendix shows that redefining the learning objective by
direct time averaging along trajectories does \emph{not} preserve
the equivalence between the unconditional and conditional objectives
established in Theorem~\ref{thm1}. Concretely, consider the following alternative
definitions of the mean-field targets:
\begin{equation}\label{appendix eq:target_wrong}
\begin{gathered}
\bm{v}(\bm{x}, t, T)=
\frac{1}{T - t}\int_t^T \bm{u}_\tau(\bm{x}_\tau)\,d\tau,\;\;\;
h(\bm{x}, t, T)=
\frac{1}{T - t}\int_t^T g_\tau(\bm{x}_\tau)\,d\tau,\\
\bm{v}(\bm{x}, t, T\mid \bm{z}) = \frac{1}{T - t}\int_t^T \bm{u}_\tau(\bm{x}_\tau \mid \bm{z})\,\mathrm{d}\tau,\;\;\;
h(\bm{x}, t, T\mid \bm{z})= \frac{1}{T - t}\int_t^T g_\tau(\bm{x}_\tau \mid \bm{z})\,\mathrm{d}\tau.
\end{gathered}
\end{equation}

\begin{proposition}[Failure of equivalence under time-averaged targets]
\label{prop:failure_time_avg}
Consider the alternative (time-averaged) targets defined in Eq.~\eqref{appendix eq:target_wrong}. Then, in general, the equivalence in Theorem~\ref{thm1} fails:
there exist a coupling distribution $q(\bm{z})$ and conditional measure paths
$\rho_t(\cdot \mid \bm{z})$ such that there is no constant
$C$ independent of $(\bm{\theta},\bm{\phi})$ for which
\[
\mathcal{L}(\bm{\theta}, \bm{\phi})
=
\mathcal{L}_{\mathrm{c}}(\bm{\theta}, \bm{\phi}) + C.
\]
Consequently, the minimization problems induced by
$\mathcal{L}_{\mathrm{c}}$ and $\mathcal{L}$ are not equivalent in general.
\end{proposition}

\begin{proof}
We begin by stating the definitions of the unconditional loss and the corresponding conditional loss.
\begin{equation*}
\begin{aligned}
\mathcal{L}(\bm{\theta},\bm{\phi})
&=\mathbb{E}_{t<T,\;\bm{x}\sim\rho_t(\bm{x})}
\Big[
\| \bm{v}_{\bm{\theta}}(\bm{x},t,T) - \bm{v}(\bm{x},t,T)\|_2^2
+ \lambda\,\| h_{\bm{\phi}}(\bm{x},t,T) - h(\bm{x},t,T)\|_2^2
\Big],\\		
\mathcal{L}_{\mathrm{c}}(\bm{\theta},\bm{\phi}) 
&= \mathbb{E}_{t<T,\bm{z}, \bm{x}\sim \rho_t(\bm{x}\vert \bm{z})}
\Big[
\| \bm{v}_{\bm{\theta}}(\bm{x},t,T) 
- \bm{v}(\bm{x},t,T\mid\bm{z})\|_2^2+ \lambda\,\| h_{\bm{\phi}}(\bm{x},t,T) 
- h(\bm{x},t,T\mid\bm{z})\|_2^2
\Big].		
\end{aligned}
\end{equation*}

Without loss of generality, we consider the velocity field $\bm{v}$;
the treatment of $h$ is analogous. The inequivalence between $\mathcal{L}$ and $\mathcal{L}_{\mathrm{c}}$ is already reflected in the cross terms related to the velocity field $\bm{v}$. All remaining terms can be treated analogously to the proof of Appendix~\ref{proof: thm1} and are therefore omitted.

We first consider the cross term appearing in the unconditional loss:
\begin{equation}\label{appendix eq: cross_l}
\begin{aligned}
&\mathbb{E}_{t<T,\;\bm{x}\sim\rho_t(\bm{x})}
\!\big\langle \bm{v}_{\bm{\theta}}(\bm{x},t,T),\, \bm{v}(\bm{x},t,T) \big\rangle \\
& = \mathbb{E}_{t<T,\;\bm{x}\sim\rho_t(\bm{x})}
\!\big \langle \bm{v}_{\bm{\theta}}(\bm{x},t,T),\, \frac{1}{T - t}\int_t^T \bm{u}_\tau(\bm{x}_\tau)\,d\tau
 \big \rangle \\
&= \mathbb{E}_{t<T}\!\int_{\mathcal{X}} \!\big \langle \bm{v}_{\bm{\theta}}(\bm{x},t,T),\, \frac{1}{T - t}\int_t^T \bm{u}_\tau(\bm{x}_\tau)\,d\tau \rho_t(\bm{x})
 \big \rangle \,\mathrm{d}\bm{x} \\
&\overset{(1)}{=} \mathbb{E}_{t<T}\!\int_{\mathcal{X}} \!\big \langle \bm{v}_{\bm{\theta}}(\bm{x},t,T),\, \frac{1}{T - t}\int_t^T \frac{\displaystyle \int_{\mathcal{Z}} \bm{u}_\tau(\bm{x}\mid \bm{z})\,\rho_\tau(\bm{x}\mid \bm{z})\,q(\bm{z})\,d\bm{z}}
       {\rho_\tau(\bm{x})}
\,d\tau \rho_t(\bm{x})
 \big \rangle \,\mathrm{d}\bm{x} \\
&\overset{(2)}{=} \mathbb{E}_{t<T}\! \int_{\mathcal{Z}} \int_{\mathcal{X}} \!\big \langle \bm{v}_{\bm{\theta}}(\bm{x},t,T),\, \frac{1}{T - t}\int_t^T \frac{\displaystyle \bm{u}_\tau(\bm{x}\mid \bm{z})\,\rho_\tau(\bm{x}\mid \bm{z})\,\,}
       {\rho_\tau(\bm{x})}
\,d\tau  \big \rangle  
 q(\bm{z})\rho_t(\bm{x}\mid \bm{z})
\,\mathrm{d}\bm{x}\,\mathrm{d}\bm{z},
\end{aligned}
\end{equation} 
where  
(1) applies the definition of the marginal flow velocity
\begin{equation*}
\begin{gathered}
\bm{u}_t(\bm{x})
= \frac{\displaystyle\int \bm{u}_t(\bm{x}\mid \bm{z})\,\rho_t(\bm{x}\mid \bm{z})\,q(\bm{z})\,d\bm{z}}
       {\rho_t(\bm{x})},\;\;\;
\rho_t(\bm{x})=\int\rho_t(\bm{x}\vert \bm{z})q(\bm{z})\mathrm{d} \bm{z}.
\end{gathered}
\end{equation*}
Then we consider the corresponding cross term in the conditional objective:
\begin{equation}\label{appendix eq: cross_cl}
\begin{aligned}
&\mathbb{E}_{t<T,\;\bm{z},\;\bm{x}\sim {\rho}_t(\bm{x}\mid\bm{z})}
\Big\langle 
\bm{v}_{\bm{\theta}}(\bm{x},t,T),\,
\bm{v}(\bm{x},t,T\mid\bm{z})
\Big\rangle\, \\[0.4em]
&= 
\mathbb{E}_{t<T,\;\bm{z},\;\bm{x}\sim {\rho}_t(\bm{x}\mid\bm{z})}
\Big\langle 
\bm{v}_{\bm{\theta}}(\bm{x},t,T),\,
\frac{1}{T-t}\int_t^T \bm{u}_\tau(\bm{x}\mid\bm{z})\,\mathrm{d}\tau
\Big\rangle\, \\[0.4em]
&= 
\mathbb{E}_{t<T}
\int_{\mathcal{Z}}\!\int_{\mathcal{X}}
\Big\langle 
\bm{v}_{\bm{\theta}}(\bm{x},t,T),\,
\frac{1}{T-t}\int_t^T \bm{u}_\tau(\bm{x}\mid\bm{z})\,\mathrm{d}\tau
\Big\rangle \,
q(\bm{z})\,\rho_t(\bm{x}\mid\bm{z})
\,\mathrm{d}\bm{x}\,\mathrm{d}\bm{z}.
\end{aligned}
\end{equation}

Comparing Eqs.~\eqref{appendix eq: cross_l} and~\eqref{appendix eq: cross_cl}, we observe that,
under the time-averaged definition, the two cross terms no longer coincide
in general. This establishes the failure of equivalence between
$\mathcal{L}$ and $\mathcal{L}_{\mathrm{c}}$.
\end{proof}

\section{Additional Results}
In this section, we present experimental details, training details, evaluation metrics, dataset descriptions, and additional results. Most of the datasets, evaluation metrics, and experimental protocols are shared with WFR-FM \citep{peng2026wfrfm}; for completeness and self-containment, we restate the relevant information here.

\subsection{Experimental Details}
All experiments were conducted on a local workstation equipped with an NVIDIA RTX 4070 Ti Super GPU and an Intel i7-12700KF CPU, except for the scalability evaluation on the 100D EB dataset, which was performed on a shared cluster with NVIDIA A100 GPUs and 128 CPU cores. 
The architecture of the neural networks for $\bm{v}_{\theta}(\bm{x}, t, T)$ and $g_{\bm{\phi}}(\bm{x}, t, T)$ are implemented using 5-layer Multilayer Perceptrons with 256 hidden units per layer and LeakyReLU activations. These networks were optimized using Pytorch \citep{pytorch}. The OET problem is solved using the Python Optimal Transport (POT) package \citep{flamary2021pot}.

\subsection{Training Details}
\paragraph{Sampling Time.}
We describe the procedure used to sample time pairs $(t, T)$ that define the temporal span for evaluating the mean fields.
In this work, time pairs are sampled from a uniform distribution.
To ensure consistency with FM and proper boundary behavior, a fraction of samples are enforced to satisfy $t = T$, corresponding to the instantaneous-velocity case, as commonly adopted in prior work~\citep{geng2025mean}. This mixture of instantaneous and interval-based samples improves training stability and generation quality.
The detailed sampling procedure is summarized in Algorithm~\ref{alg:time_sampling}.

\begin{algorithm}[h]
\caption{Sampling time pairs $(t, T)$}
\label{alg:time_sampling}
\begin{algorithmic}[1]
\STATE \textbf{Input:} batch size $B$, proportion $p_{\text{diff}}$ of samples with $t \neq T$
\STATE \textbf{for each sample} do
\STATE \quad Draw $z \sim \mathrm{Bernoulli}(1 - p_{\text{diff}})$
\IF{$z = 1$} 
   \STATE Sample $t \sim \mathcal{U}(0,1)$
   \STATE $T \gets t$
\ELSE
   \STATE Sample $t_1, t_2 \sim \mathcal{U}(0,1)$
   \STATE $t \gets \min(t_1, t_2)$,\quad $T \gets \max(t_1, t_2)$
\ENDIF
\STATE \textbf{return} $\{(t, T)\}_{i=1}^{B}$
\end{algorithmic}
\end{algorithm}

\subsection{Evaluation Metrics}
\label{metrics}
We evaluate model performance using two metrics: the 1-Wasserstein distance ($\mathcal{W}_1$),  which measures the similarity between predicted and true distributions, and the Relative Mass Error (RME), which assesses how well the model captures cell population growth. The metrics are defined as:
\begin{equation*}
\begin{gathered}
	\mathcal{W}_1(p, q) = \min_{\pi \in \Pi(p,q)} \int \|\bm{x} - \bm{y}\|_2 d\pi(\bm{x}, \bm{y}),\\
	\text{RME}(t_k) = \frac{\left| \sum_i{w_i(t_k)} - n_k/n_0 \right|}{n_k/n_0}.
\end{gathered}	
\end{equation*}
Here, $p$ and $q$ denote the empirical distributions of predicted and observed cells, respectively, $w_i(t_k)$ represents the inferred mass associated with cell $i$ at time $t_k$, and $n_k$ denotes the number of observed cells at time point $k$.

For evaluation, we propagate the learned dynamics from the initial cell population, where all cells are initialized with equal weights $w_i(0)=1/n_0$, to generate predicted cell states at later time points. When an unbalanced formulation is employed, cell weights are evolved jointly with the state dynamics; otherwise, weights remain uniform throughout. We then compute the weighted $\mathcal{W}_1$ distance and the RME by comparing the predicted distributions against the observed data at each time point.
For selected datasets, we further conduct a hold-out evaluation by excluding one time point during training and reporting the $\mathcal{W}_1$ distance on the unseen snapshot. To ensure a fair comparison across methods, we reimplemented TIGON \citep{TIGON} to mitigate numerical instabilities observed in the original implementation. For the remaining baselines, we largely follow the default configurations reported in their respective papers when applicable; otherwise, we adjust network widths to match parameter scales and tune training epochs and learning rates for each dataset to ensure balanced comparisons.

\subsection{Results on held-out time points}
We report quantitative results on held-out time points for real biological datasets, including EMT, EB, CITE, and Mouse.
For each dataset, we perform an evaluation over intermediate time points by holding out one time point at a time for testing, while using all remaining time points for training.
Performance is evaluated on the held-out time point using the $\mathcal{W}_1$ metric, and the final results are obtained by averaging over all intermediate time points.
The mean $\mathcal{W}_1$ scores are summarized in Table~\ref{tab:Q22}.

\label{Appendix: held-out}
\begin{table}[H]
  \caption{Mean $\mathcal{W}_1$ over held-out time points on EMT, EB, CITE, and Mouse datasets.}
  \label{tab:Q22}

  \begin{center}
    \begin{small}
      \begin{sc}
        \begin{tabular}{lcccc}
          \toprule
          \textbf{Method} & EMT (10D) & EB (50D) & CITE (50D) & Mouse (50D) \\
          \midrule
          MMFM      & 0.323                 & 11.213                & 38.521                 & 8.263 \\
          Metric FM & 0.314                 & 10.726                & 37.342                 & 7.753 \\
          SF2M      & 0.308\tiny$\pm$0.001   & 10.986\tiny$\pm$0.006  & 38.333\tiny$\pm$0.002   & 8.646\tiny$\pm$0.004 \\
          MIOFlow   & 0.325                 & 10.960                & 39.574                 & 7.779 \\
          TIGON     & 0.360                 & 11.080                & 38.159                 & 6.868 \\
          DeepRUOT  & 0.323\tiny$\pm$0.002   & \textbf{10.075}\tiny$\pm$0.004 & 37.892\tiny$\pm$0.002 & 6.847\tiny$\pm$0.003 \\
          Var-RUOT  & 0.320\tiny$\pm$0.003   & 11.035\tiny$\pm$0.017 & 38.393\tiny$\pm$0.029  & 8.672\tiny$\pm$0.040 \\
          UOT-FM    & 0.322                 & 11.344                & 38.649                 & 9.332 \\
          VGFM      & 0.301                 & 10.370                & 37.386                 & 8.496 \\
          WFR-FM    & \textbf{0.298}        & 10.157                & \underline{37.221}     & \textbf{6.586} \\
          \textbf{WFR-MFM} 
                    & \underline{0.299}     & \underline{10.135}     & \textbf{35.736}        & \underline{6.714} \\
          \bottomrule
        \end{tabular}
      \end{sc}
    \end{small}
  \end{center}

  \vskip -0.1in
\end{table}
\subsection{Performance on Simulation Gene Dataset}
Following the experimental setup in \citep{DeepRUOT}, we utilize a synthetic dataset representing a gene regulatory network. The temporal evolution of gene concentrations is modeled by a system of stochastic differential equations (SDEs) as follows:
$$\frac{dX_1}{dt} = \frac{\alpha_1 X_1^2 + \beta}{1 + \alpha_1 X_1^2 + \gamma_2 X_2^2 + \gamma_3 X_3^2 + \beta} - \delta_1 X_1 + \eta_1 \xi_t$$$$\frac{dX_2}{dt} = \frac{\alpha_2 X_2^2 + \beta}{1 + \gamma_1 X_1^2 + \alpha_2 X_2^2 + \gamma_3 X_3^2 + \beta} - \delta_2 X_2 + \eta_2 \xi_t$$$$\frac{dX_3}{dt} = \frac{\alpha_3 X_3^2}{1 + \alpha_3 X_3^2} - \delta_3 X_3 + \eta_3 \xi_t$$

In this system, $X_i(t)$ denotes the concentration level of the $i$-th gene. The network topology incorporates a mutual inhibition mechanism between genes $X_1$ and $X_2$, both of which exhibit self-activation capabilities. Furthermore, an external signal $\beta$ promotes the activation of $X_1$ and $X_2$, while gene $X_3$ acts as a repressor for both. The parameters $\alpha_i$, $\gamma_i$, and $\delta_i$ correspond to the rates of self-activation, cross-inhibition, and degradation, respectively, with $\eta_i \xi_t$ representing the stochastic noise component. 

To simulate population dynamics, we incorporate a probabilistic cell division process. The instantaneous growth rate $g$ is modulated by the expression level of $X_2$, defined as \(g = \alpha_g \frac{X_2^2}{1 + X_2^2}\). When division occurs, the parent cell's state is passed to the daughter cells with minor random perturbations. The dataset consists of snapshots taken at discrete time intervals $t \in \{0, 8, 16, 24, 32\}$, originating from two separate initial populations: one undergoing dynamic transition and growth, and the other maintaining a steady-state equilibrium.

\textbf{Choice of WFR-penalty $\delta$}. The hyperparameter $\delta$ serves as a weighting factor in the WFR metric, balancing the trade-off between the transport cost and the unbalanced mass variation cost (birth-death). As defined in equation  \ref{eq:wfr_dynamic}, a higher $\delta$ imposes a stronger penalty on the growth term $g$, thereby forcing the model to prioritize spatial transport over mass creation or annihilation. Table \ref{tab:sensitivity_delta} (evaluated on the Simulation Gene dataset) shows that performance is highly sensitive to $\delta$, with the optimum found at $\delta=1.5$. Increasing $\delta$ degrades results significantly; extreme values (e.g., $\delta=20$) cause drastic divergence ($\mathcal{W}_1 > 110$). This suggests the data involves significant mass variation, and excessively penalizing the birth-death term forces the model into erroneous pure-transport solutions. We set $\delta=1.5$ for the main experiments on this dataset.

\begin{table}[htbp]
    \centering
    \caption{Sensitivity analysis for parameter $\delta$ on Simulation Gene Dataset ($p_{\text{diff}}=0.6, \lambda=0.05$).}
    \label{tab:sensitivity_delta}
    {
    \begin{tabular}{lcccccccc}
        \toprule
        \textbf{Parameter} & \multicolumn{2}{c}{\textbf{t=1}} & \multicolumn{2}{c}{\textbf{t=2}} & \multicolumn{2}{c}{\textbf{t=3}} & \multicolumn{2}{c}{\textbf{t=4}} \\
        \cmidrule(lr){2-3} \cmidrule(lr){4-5} \cmidrule(lr){6-7} \cmidrule(lr){8-9}
         & $\mathcal{W}_1$ & TMV & $\mathcal{W}_1$ & TMV & $\mathcal{W}_1$ & TMV & $\mathcal{W}_1$ & TMV \\
        \midrule
        $\delta= 1$   & 0.0234 & 0.0061 & 0.0250 & 0.0086 & 0.0200 & 0.0060 & 0.0194 & 0.0057 \\
        $\delta= 1.5$ & 0.0224 & 0.0006 & 0.0219 & 0.0006 & 0.0212 & 0.0054 & 0.0208 & 0.0062 \\
        $\delta= 2$   & 0.0400 & 0.0005 & 0.0514 & 0.0067 & 0.0452 & 0.0133 & 0.0566 & 0.0133 \\
        $\delta= 5$   & 0.0469 & 0.0002 & 0.1579 & 0.0053 & 0.1251 & 0.0283 & 0.1576 & 0.0288 \\
        $\delta= 10$  & 0.0607 & 0.0001 & 0.1065 & 0.0114 & 0.1058 & 0.0218 & 0.2655 & 0.0164 \\
        $\delta= 20$  & 0.1099 & 0.0001 & 0.2002 & 0.0081 & 0.7320 & 0.0289 & 110.88 & 0.1213 \\
        \bottomrule
    \end{tabular}%
    }
\end{table}

\textbf{Choice of Cross-time Sampling Proportion $p_{\text{diff}}$}. The parameter $p_{\text{diff}}$ controls the probability of sampling different time pairs ($t \neq T$). A balanced $p_{\text{diff}}$ is essential: excessively low values restrict learning to instantaneous fields, failing to capture time-averaged dynamics, while overly high values may neglect immediate time-step reconstruction. Table \ref{tab:sensitivity_pdiff} presents the sensitivity analysis. We observe that a low $p_{\text{diff}}=0.1$ results in suboptimal performance with higher transport errors. However, performance improves significantly as $p_{\text{diff}}$ increases, reaching an optimal range between 0.4 and 0.6 where both Wasserstein distances and TMV are minimized. Notably, setting $p_{\text{diff}}$ too high (e.g., 0.8) leads to slight degradation in long-term accuracy ($t=3, 4$). Based on these results, we set $p_{\text{diff}}=0.6$ for the main experiments on this dataset.

\begin{table}[htbp]
    \centering
    \caption{Sensitivity analysis for parameter $p_{\text{diff}}$ on Simulation Gene Dataset ($\delta= 1.5$, $\lambda=0.05$).}
    \label{tab:sensitivity_pdiff}
    {
    \begin{tabular}{lcccccccc}
        \toprule
        \textbf{Parameter} & \multicolumn{2}{c}{\textbf{t=1}} & \multicolumn{2}{c}{\textbf{t=2}} & \multicolumn{2}{c}{\textbf{t=3}} & \multicolumn{2}{c}{\textbf{t=4}} \\
        \cmidrule(lr){2-3} \cmidrule(lr){4-5} \cmidrule(lr){6-7} \cmidrule(lr){8-9}
         & $\mathcal{W}_1$ & TMV & $\mathcal{W}_1$ & TMV & $\mathcal{W}_1$ & TMV & $\mathcal{W}_1$ & TMV \\
        \midrule
        $p_{\text{diff}}= 0.1$ & 0.024 & 0.0034 & 0.027 & 0.0038 & 0.029 & 0.0069 & 0.030 & 0.0135 \\
        $p_{\text{diff}}= 0.2$ & 0.023 & 0.0004 & 0.027 & 0.0038 & 0.026 & 0.0086 & 0.026 & 0.0121 \\
        $p_{\text{diff}}= 0.3$ & 0.023 & 0.0019 & 0.026 & 0.0051 & 0.024 & 0.0081 & 0.024 & 0.0096 \\
        $p_{\text{diff}}= 0.4$ & 0.022 & 0.0007 & 0.024 & 0.0028 & 0.024 & 0.0095 & 0.023 & 0.0125 \\
        $p_{\text{diff}}= 0.5$ & 0.023 & 0.0008 & 0.023 & 0.0030 & 0.022 & 0.0083 & 0.022 & 0.0092 \\
        $p_{\text{diff}}= 0.6$ & 0.021 & 0.0003 & 0.022 & 0.0010 & 0.020 & 0.0052 & 0.019 & 0.0061 \\
        $p_{\text{diff}}= 0.7$ & 0.023 & 0.0012 & 0.022 & 0.0011 & 0.023 & 0.0046 & 0.022 & 0.0075 \\
        $p_{\text{diff}}= 0.8$ & 0.022 & 0.0015 & 0.023 & 0.0022 & 0.025 & 0.0100 & 0.024 & 0.0118 \\
        \bottomrule
    \end{tabular}%
    }
\end{table}

\textbf{Choice of Loss Weight $\lambda$}. The parameter $\lambda$ balances the supervision between the average velocity field and the mass-growth field. An appropriate $\lambda$ is critical: a value too low weakens growth supervision, risking incorrect population estimation, while a value too high dominates the loss, sacrificing spatial transport accuracy. As shown in Table \ref{tab:sensitivity_lambda}, the model achieves optimal performance at $\lambda=0.05$, exhibiting the lowest transport errors ($\mathcal{W}_1$) and trajectory variance (TMV), particularly at later time steps ($t=3, 4$). Notably, the algorithm demonstrates significant robustness; performance metrics remain highly stable across a wide range of magnitudes ($\lambda \in [1, 50]$), indicating that the method is insensitive to hyperparameter tuning provided $\lambda$ is not negligible. Consequently, we fix $\lambda=0.05$ for the reported experiments.

\begin{table}[htbp]
    \centering
    \caption{Sensitivity analysis for parameter $\lambda$ on simulation gene dataset ($\delta= 1.5$, $p_{\text{diff}}=0.6$).}
    \label{tab:sensitivity_lambda}
    {%
    \begin{tabular}{lcccccccc}
        \toprule
        \textbf{Parameter} & \multicolumn{2}{c}{\textbf{t=1}} & \multicolumn{2}{c}{\textbf{t=2}} & \multicolumn{2}{c}{\textbf{t=3}} & \multicolumn{2}{c}{\textbf{t=4}} \\
        \cmidrule(lr){2-3} \cmidrule(lr){4-5} \cmidrule(lr){6-7} \cmidrule(lr){8-9}
         & $\mathcal{W}_1$ & TMV & $\mathcal{W}_1$ & TMV & $\mathcal{W}_1$ & TMV & $\mathcal{W}_1$ & TMV \\
        \midrule
        $\lambda = 0.01$ & 0.023 & 0.0005 & 0.022 & 0.0001 & 0.024 & 0.0081 & 0.023 & 0.0100 \\
        $\lambda = 0.05$ & 0.021 & 0.0003 & 0.022 & 0.0010 & 0.020 & 0.0052 & 0.019 & 0.0061 \\
        $\lambda = 0.1$  & 0.022 & 0.0010 & 0.022 & 0.0013 & 0.023 & 0.0085 & 0.022 & 0.0109 \\
        $\lambda = 1$    & 0.022 & 0.0004 & 0.023 & 0.0015 & 0.023 & 0.0076 & 0.022 & 0.0087 \\
        $\lambda = 10$   & 0.022 & 0.0006 & 0.022 & 0.0003 & 0.023 & 0.0076 & 0.022 & 0.0103 \\
        $\lambda = 50$   & 0.023 & 0.0012 & 0.022 & 0.0007 & 0.023 & 0.0073 & 0.022 & 0.0091 \\
        \bottomrule
    \end{tabular}%
    }
\end{table}

\subsection{Performance on Dyngen Dataset}
We adopt the unbalanced bifurcation simulation previously analyzed in \citep{mioflow} and \citep{wang2025joint}. Produced by Dyngen \citep{cannoodt2021spearheading}, this dataset comprises 728 cells with dimensionality reduced to 5 via PHATE \citep{moon2019visualizing}. The complexity of this data arises from two main factors: significant fluctuations in total mass over time, and a pronounced structural asymmetry, characterized by a much larger number of cells populating the lower branch compared to the upper one. We fix WFR penalty to $\delta = 1.5,$ cross-time sampling proportion to $p_{\text{diff}}=0.5$ and growth penalty to $\lambda=5$. As shown in Figure~\ref{fig:Dyngen_results}, WFR-FM accurately captures the underlying dynamics.

Detailed quantitative comparisons are provided in Table \ref{tab:Dyngen_metrics}. Our method yields highly competitive results compared to WFR-FM. Notably, we achieve the lowest $\mathcal{W}_1$ error at $t=3$ and consistently maintain the second-lowest errors across other time points for both $\mathcal{W}_1$ and RME.
\begin{figure}[H]
    \centering
        \includegraphics[width=\linewidth]{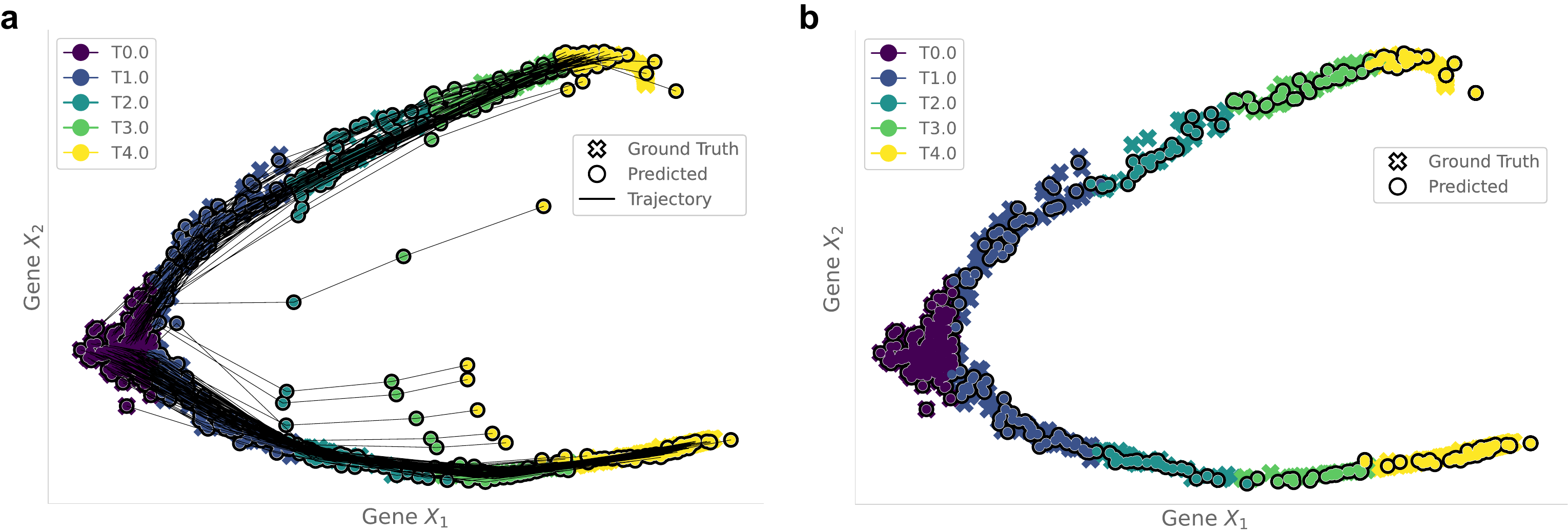}
\caption{
      \textbf{Learned dynamics on the Dyngen dataset.}
      (a) Trajectories of individual cells with a fixed cell count, where only cell weights evolve over time.
      (b) Population-level dynamics obtained by normalizing weights into probabilities and resampling cells, revealing changes in cell abundance.
    }\label{fig:Dyngen_results}
\end{figure}

\begin{table}[htbp]
\centering
\caption{Comparison of method performance over time on the Dyngen dataset. }
\label{tab:Dyngen_metrics}
{%
\begin{tabular}{lcccccccc}
\toprule
\textbf{Method} & \multicolumn{2}{c}{\textbf{t=1}} & \multicolumn{2}{c}{\textbf{t=2}} & \multicolumn{2}{c}{\textbf{t=3}} & \multicolumn{2}{c}{\textbf{t=4}} \\
\cmidrule(lr){2-3} \cmidrule(lr){4-5} \cmidrule(lr){6-7} \cmidrule(lr){8-9}
 & $\mathcal{W}_1$ & RME & $\mathcal{W}_1$ & RME & $\mathcal{W}_1$ & RME & $\mathcal{W}_1$ & RME \\
\midrule
MMFM       & 0.574 & ---   & 1.704 & ---   & 1.499 & ---   & 1.706 & ---   \\
Metric FM  & 0.892 & ---   & 2.347 & ---   & 2.030 & ---   & 1.799 & ---   \\
SF2M       & 0.637 & ---   & 1.266 & ---   & 1.415 & ---   & 1.790 & ---   \\
MIOFlow    & 0.420 & ---   & 0.640 & ---   & 1.537 & ---   & 1.263 & ---   \\
TIGON      & 0.446 & 0.033 & 0.584 & 0.060 & 0.415 & 0.023 & 0.603 & 0.071 \\
DeepRUOT   & 0.454 & 0.011 & 0.481 & 0.070 & 0.870 & 0.104 & 0.688 & 0.074 \\
Var-RUOT   & 0.315 & 0.128 & 0.548 & 0.336 & 0.630 & 0.222 & 0.593 & \underline{0.023} \\
UOT-FM     & 0.652 & 0.008 & 0.780 & 0.077 & 1.252 & 0.090 & 2.130 & 0.213 \\
VGFM       & 0.335 & \textbf{0.001} & 0.312 & 0.073 & 1.109 & 0.041 & 0.634 & 0.033 \\
WFR-FM     & \textbf{0.110} & \underline{0.003} & \textbf{0.098} & \textbf{0.007} & \underline{0.211} & \textbf{0.008} & \textbf{0.121} & \textbf{0.002} \\
WFR-MFM & \underline{0.130} & 0.007 & \underline{0.163} & \underline{0.009} & \textbf{0.202} & \underline{0.017} & \underline{0.197} & 0.111 \\
\bottomrule
\end{tabular}%
}
\end{table}

\subsection{Performance on Gaussian Mixture Dataset}
We evaluate our method on the 1000-dimensional Gaussian Mixture dataset adopted from ~\citep{wang2025joint}. Following their experimental setup, we generate an initial distribution of 500 samples (100 from an upper Gaussian component and 400 from a lower Gaussian component) and a target distribution of 1,400 samples (1,000 from the upper component and 200 from each of the two lower Gaussians). This configuration is designed to simulate unbalanced population dynamics, specifically modeling cell proliferation in the upper region, which serves as a robust benchmark for evaluating transport algorithms in high-dimensional settings.

We apply WFR-MFM to this task with hyperparameters set to $\delta = 1.4,$ $p_{\text{diff}}=0.05$ and $\lambda=1$. As illustrated in Figure~\ref{fig:gaussian_mixture}, our method successfully captures the underlying dynamics.

\begin{figure}[h]
    \centering
        \includegraphics[width=\linewidth]{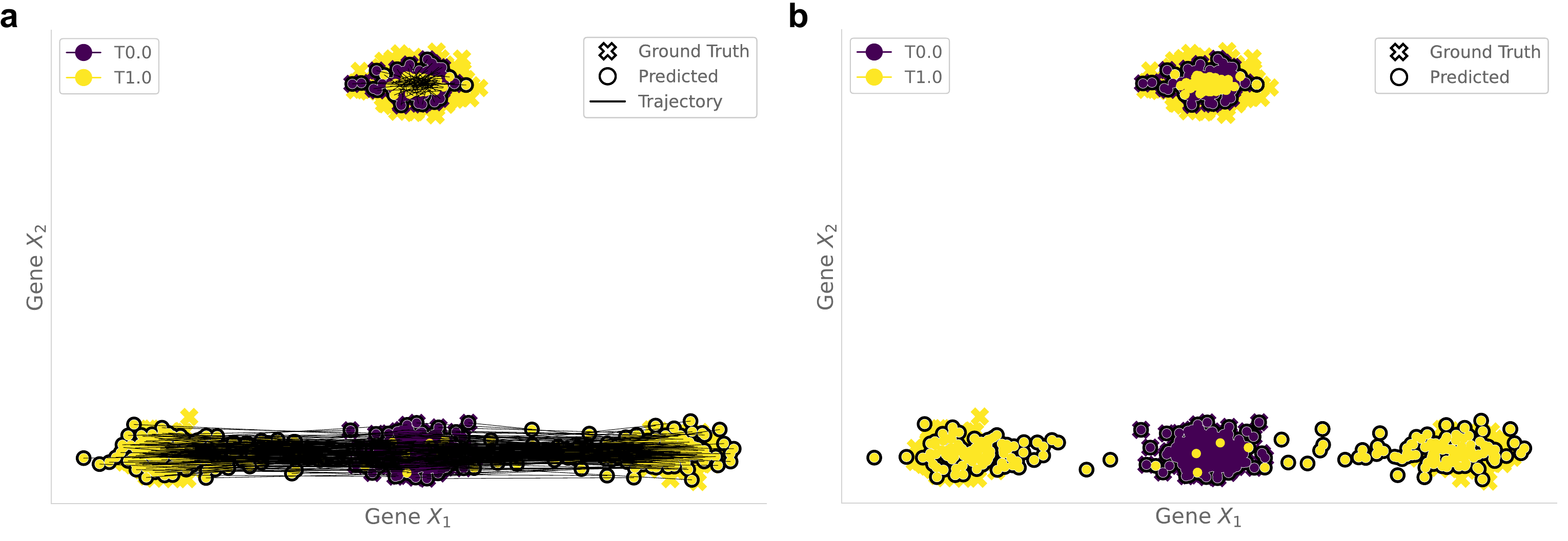}
\caption{
      \textbf{Learned dynamics on the Gaussian 1000D dataset.}
      (a) Trajectories of individual cells with a fixed cell count, where only cell weights evolve over time.
      (b) Population-level dynamics obtained by normalizing weights into probabilities and resampling cells, revealing changes in cell abundance.
    }        \label{fig:gaussian_mixture}
\end{figure}

\subsection{Performance on Epithelial--Mesenchymal Transition Dataset}
We utilize the single-cell dataset capturing the epithelial-mesenchymal transition (EMT) in A549 lung cancer cells, originally collected by \citep{cook2020context}. This dataset comprises samples taken at four distinct time points during the transition process. Consistent with the preprocessing steps outlined in \citep{TIGON}, the data dimensionality was reduced to 10 using an AutoEncoder.

We evaluate WFR-MFM on this dataset with hyperparameters set to  $\delta = 2,$ $p_{\text{diff}}=0.05$ and $\lambda=20$. The quantitative results are summarized in Table~\ref{tab:emt}. WFR-MFM demonstrates highly competitive performance across all time intervals. In the distribution matching task (measured by $\mathcal{W}_1$), our method consistently ranks as the second-best approach, outperforming other unbalanced transport baselines such as VGFM and UOT-FM, and trailing only the WFR-FM benchmark. Furthermore, WFR-MFM maintains a low RME, indicating its robustness in modeling the mass variation inherent in the EMT process where cells exhibit enhanced stemness and proliferation. 

\begin{table}[t]
\centering
\caption{Comparison of method performance over time on the 10D EMT dataset. }
\label{tab:emt}
\begin{tabular}{lcccccc}
\toprule
\textbf{Method} & \multicolumn{2}{c}{\textbf{t=1}} & \multicolumn{2}{c}{\textbf{t=2}} & \multicolumn{2}{c}{\textbf{t=3}} \\
\cmidrule(lr){2-3} \cmidrule(lr){4-5} \cmidrule(lr){6-7}
& $\mathcal{W}_1$ & RME & $\mathcal{W}_1$ & RME & $\mathcal{W}_1$ & RME \\
\midrule
MMFM & 0.2576 & --- & 0.2874 & --- & 0.3102 & --- \\
Metric FM & 0.2605 & --- & 0.2971 & --- & 0.3050 & --- \\
SF2M & 0.2566 & --- & 0.2811 & --- & 0.2900 & --- \\
MIOFlow & 0.2439 & --- & 0.2665 & --- & 0.2841 & --- \\
TIGON & 0.2433 & \underline{0.002} & 0.2661 & \underline{0.003} & 0.2847 & \textbf{0.001} \\
DeepRUOT & 0.2902 & \textbf{0.001} & 0.3193 & 0.011 & 0.3291 & \underline{0.002} \\
Var-RUOT & 0.2540 & 0.075 & 0.2670 & 0.014 & 0.2683 & 0.041 \\
UOT-FM & 0.2538 & \underline{0.002} & 0.2696 & 0.013 & 0.2771 & 0.010 \\
VGFM & 0.2350 & 0.016 & 0.2420 & 0.011 & 0.2450 & 0.018 \\
WFR-FM & \textbf{0.2099} & \textbf{0.001} & \textbf{0.2272} & \textbf{0.002} & \textbf{0.2346} & \textbf{0.001} \\
WFR-MFM & \underline{0.2250} & 0.005 & \underline{0.2376} & \underline{0.003} & \underline{0.2440} & 0.003 \\
\bottomrule
\end{tabular}
\end{table}

\subsection{Performance on Embryoid Bodies Dataset}
\label{Appendix:EB}
Our study employs the human embryoid body (EB) differentiation dataset from \citep{moon2019visualizing}, which captures 16,819 cells sampled at five intervals over a 27-day period to model early development. To prepare the data for trajectory inference, we adopt the dimensionality reduction strategy used in \citep{wang2025joint}, compressing the original gene expression space using Principal Component Analysis (PCA). This preprocessed, lower-dimensional representation serves as the direct input for our WFR-MFM algorithm.

\paragraph{Large-scale Experiment.}
Figure~\ref{fig:compare} provides complementary views of the efficiency--accuracy trade-offs on the 100D EB dataset,
illustrating the relationships between predictive accuracy ($\mathcal{W}_1$ distance) and training time,
as well as inference time, with color indicating GPU memory usage.

\begin{figure}[t]
  \centering
  \begin{subfigure}[t]{0.48\columnwidth}
    \centering
    \includegraphics[width=\linewidth]{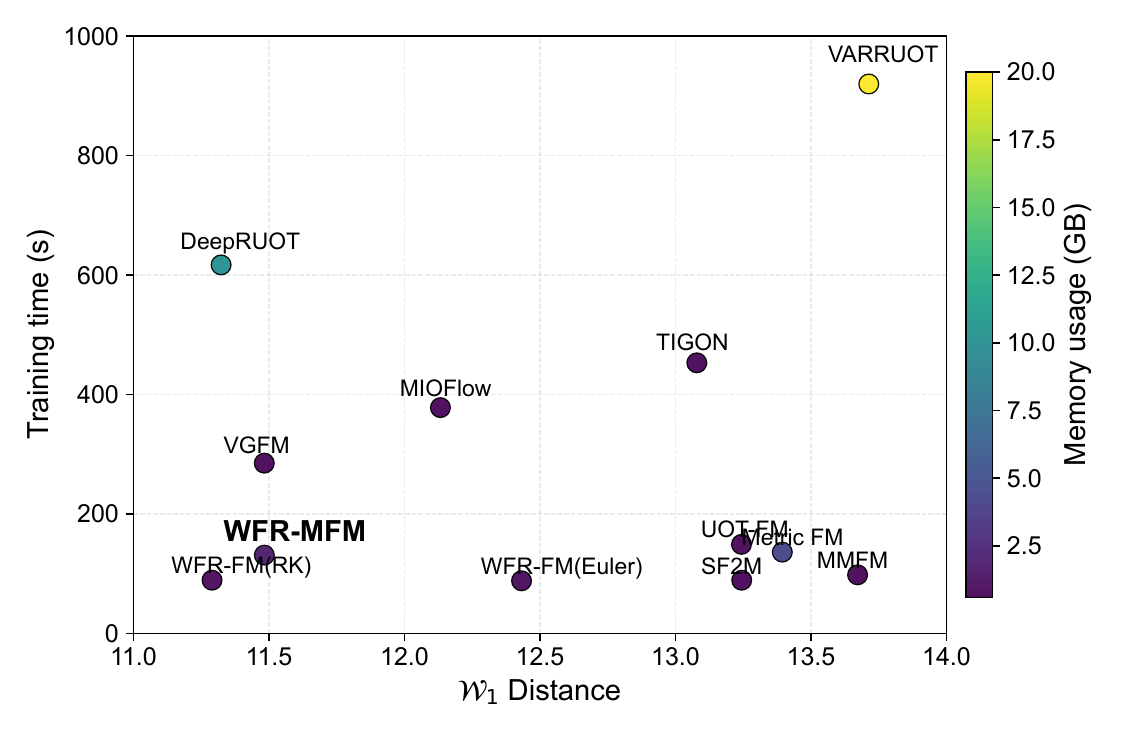}
    \caption{}
  \end{subfigure}
  \hfill
  \begin{subfigure}[t]{0.48\columnwidth}
    \centering
    \includegraphics[width=\linewidth]{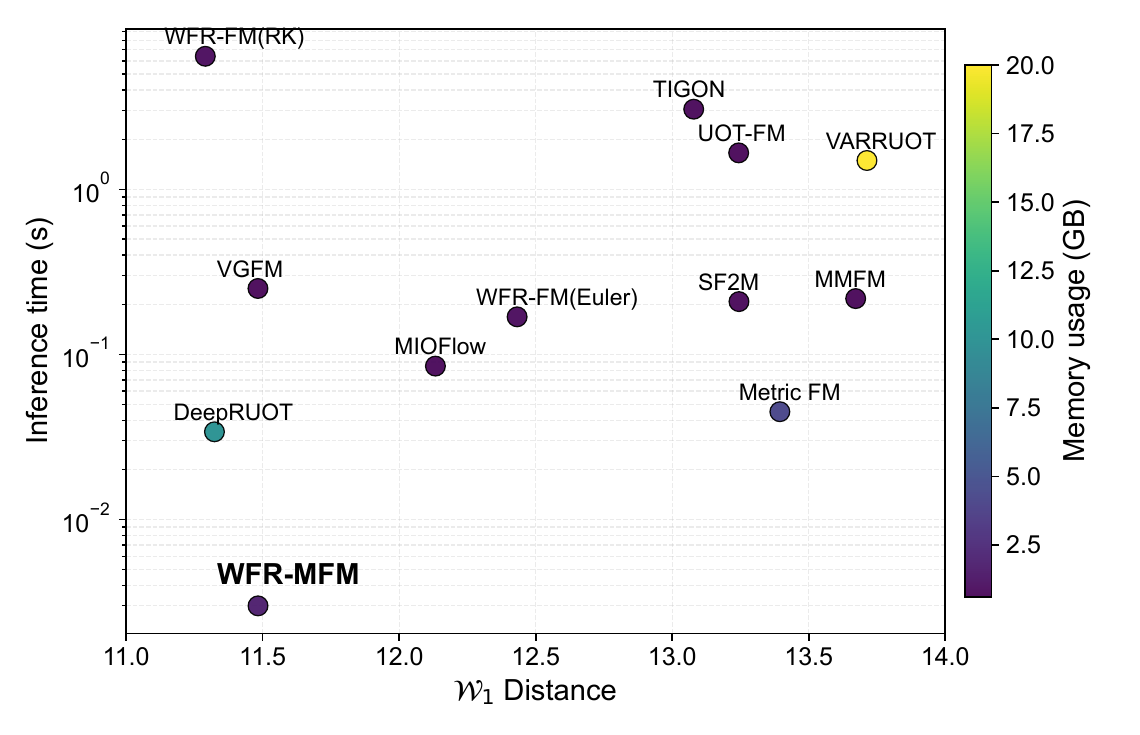}
    \caption{}
  \end{subfigure}
\caption{
  \textbf{Efficiency on the 100D EB dataset.}
  Methods are compared in terms of predictive accuracy ($\mathcal{W}_1$ distance)
  and computational efficiency, including training cost and inference latency,
  with color indicating GPU memory usage.
  \textit{WFR-FM(RK)} and \textit{WFR-FM(Euler)} correspond to inference using an adaptive Dormand--Prince RK5(4) solver and an explicit Euler solver (100 steps), respectively.
}
  \label{fig:compare}
\end{figure}

\paragraph{Efficiency of Mini-batch WFR-OET.}
We utilized a mini-batch strategy to improve training efficiency. Table~\ref{tab:minibatch} shows that this method significantly cuts down training time with negligible impact on accuracy. The model demonstrates robustness across batch sizes, with the 1500-batch setting achieving an even better RME (0.009) than the full-batch baseline (0.012). Therefore, we selected 1500 as the optimal batch size, providing the best combination of low training cost, high mass conservation, and reliable $\mathcal{W}_1$ performance.

\begin{table}[t]
\centering
\caption{Sensitivity analysis for batch size of mini-batch WFR-MFM on the 100D EB dataset. The time reported represents the total training time.}
\label{tab:minibatch}
 \resizebox{\textwidth}{!}{%
\begin{tabular}{lccccccccccc} 
\toprule
\textbf{Batch Size} & \multicolumn{2}{c}{\textbf{t=1}} & \multicolumn{2}{c}{\textbf{t=2}} & \multicolumn{2}{c}{\textbf{t=3}} & \multicolumn{2}{c}{\textbf{t=4}} & \multicolumn{2}{c}{\textbf{Average}} & \textbf{Train Time (s)} \\
\cmidrule(lr){2-3} \cmidrule(lr){4-5} \cmidrule(lr){6-7} \cmidrule(lr){8-9} \cmidrule(lr){10-11} 
 & $\mathcal{W}_1$ & RME & $\mathcal{W}_1$ & RME & $\mathcal{W}_1$ & RME & $\mathcal{W}_1$ & RME & $\mathcal{W}_1$ & RME &  \\
 \midrule
500  & 10.072 & 0.007 & 11.248 & 0.005 & 11.748 & 0.004 & 12.815 & 0.009 & 11.471 & 0.006 & 136\\
1000 & 10.025 & 0.008 & 11.294 & 0.011 & 11.803 & 0.007 & 12.831 & 0.018 & 11.488 & 0.011 & 131\\
1500 & 10.020 & 0.007 & 11.274 & 0.007 & 11.779 & 0.006 & 12.859 & 0.015 & 11.483 & 0.009 & 131\\
2000 & 10.050 & 0.008 & 11.289 & 0.012 & 11.858 & 0.011 & 12.926 & 0.018 & 11.531 & 0.012 & 135\\
3000 & 9.999  & 0.008 & 11.283 & 0.009 & 11.792 & 0.012 & 12.892 & 0.022 & 11.492 & 0.013 & 151\\
\midrule
w/o mini-batch & 9.995 & 0.009 & 11.222 & 0.012 & 11.720 & 0.010 & 12.782 & 0.017 & 11.430 & 0.012 & 295\\
\bottomrule
\end{tabular}%
}
\end{table}

\paragraph{Scalability of WFR-MFM across Dimensions.}
Utilizing the optimal mini-batch size of 1500 established previously, we evaluated the scalability of WFR-MFM by comparing it against baseline methods across varying data dimensions: 5 (Table~\ref{tab:5d_eb}), 50 (Table~\ref{tab:50d_eb}), and 100 (Table~\ref{tab:100d_eb}). In these experiments, we fixed $\delta = 25$ and $p_{\text{diff}}=0.6$, with $\lambda$ adjusted to 50, 100, and 1, respectively. The results indicate that WFR-MFM consistently outperforms competing approaches at the majority of time points across all datasets. 

\begin{table}[H]
    \centering
    \caption{Comparison of method performance over time on the 5D EB dataset. }
    \label{tab:5d_eb}
   {%
    \begin{tabular}{lcccccccc}
        \toprule
        \multirow{2}{*}{\textbf{Method}} & \multicolumn{2}{c}{\textbf{t=1}} & \multicolumn{2}{c}{\textbf{t=2}} & \multicolumn{2}{c}{\textbf{t=3}} & \multicolumn{2}{c}{\textbf{t=4}} \\
        \cmidrule(lr){2-3} \cmidrule(lr){4-5} \cmidrule(lr){6-7} \cmidrule(lr){8-9}
        & $\mathcal{W}_1$ & RME & $\mathcal{W}_1$ & RME & $\mathcal{W}_1$ & RME & $\mathcal{W}_1$ & RME \\
        \midrule
        MMFM       & 0.477 & ---   & 0.554 & ---   & 0.781 & ---   & 0.872 & --- \\
        Metric FM  & 0.449 & ---   & 0.552 & ---   & 0.583 & ---   & 0.597 & --- \\
        SF2M       & 0.556 & ---   & 0.715 & ---   & 0.750 & ---   & 0.650 & --- \\
        MIOFlow    & 0.442 & ---   & 0.585 & ---   & 0.651 & ---   & 0.670 & --- \\
        TIGON      & 0.386 & \underline{0.002} & 0.502 & 0.015 & 0.602 & 0.021 & 0.600 & 0.027 \\
        DeepRUOT   & 0.386 & 0.005 & 0.497 & 0.017 & 0.591 & 0.021 & 0.585 & 0.030 \\
        Var-RUOT   & 0.416 & 0.111 & 0.486 & 0.144 & 0.509 & 0.054 & 0.511 & 0.022 \\
        UOT-FM     & 0.544 & 0.032 & 0.670 & 0.029 & 0.729 & 0.016 & 0.852 & 0.041 \\
        VGFM       & 0.402 & 0.046 & 0.494 & 0.018 & 0.525 & 0.035 & 0.573 & 0.021 \\
        WFR-FM     & \textbf{0.324} & 0.003 & \textbf{0.401} & \textbf{0.001} & \textbf{0.431} & \underline{0.005} & \underline{0.510} & \underline{0.005} \\
        WFR-MFM    & \underline{0.356} & \textbf{0.0001}& \underline{0.438} & \underline{0.010} & \underline{0.477} & \textbf{0.0002}& \textbf{0.502} & \textbf{0.001} \\
        \bottomrule
    \end{tabular}
    }
\end{table}

\begin{table}[H]
    \centering
    \caption{Comparison of method performance over time on the 50D EB dataset.}
    \label{tab:50d_eb}
    {%
    \begin{tabular}{lcccccccc}
        \toprule
        \multirow{2}{*}{\textbf{Method}} & \multicolumn{2}{c}{\textbf{t=1}} & \multicolumn{2}{c}{\textbf{t=2}} & \multicolumn{2}{c}{\textbf{t=3}} & \multicolumn{2}{c}{\textbf{t=4}} \\
        \cmidrule(lr){2-3} \cmidrule(lr){4-5} \cmidrule(lr){6-7} \cmidrule(lr){8-9}
        & $\mathcal{W}_1$ & RME & $\mathcal{W}_1$ & RME & $\mathcal{W}_1$ & RME & $\mathcal{W}_1$ & RME \\
        \midrule
        MMFM       & 9.124 & ---   & 10.474 & ---   & 11.022 & ---   & 11.480 & --- \\
        Metric FM  & 8.506 & ---   & 9.795  & ---   & 10.621 & ---   & 12.042 & --- \\
        SF2M       & 9.247 & ---   & 10.882 & ---   & 11.650 & ---   & 12.154 & --- \\
        MIOFlow    & 8.447 & ---   & 9.229  & ---   & 9.436  & ---   & 10.123 & --- \\
        TIGON      & 8.433 & 0.067 & 9.275  & 0.022 & 9.802  & 0.179 & 10.148 & 0.101 \\
        DeepRUOT   & 8.169 & \underline{0.003} & 9.049  & 0.038 & 9.378  & 0.088 & 9.733  & \underline{0.004} \\
        Var-RUOT   & 9.442 & 0.128 & 9.709  & 0.081 & 10.482 & 0.031 & 10.735 & 0.030 \\
        UOT-FM     & 8.717 & 0.063 & 10.858 & 0.009 & 11.813 & 0.022 & 12.733 & 0.018 \\
        VGFM       & 7.951 & 0.089 & 8.747  & 0.042 & 9.244  & 0.019 & \underline{9.620}  & 0.044 \\
        WFR-FM     & \underline{7.664} & 0.008 & \underline{8.659}  & \underline{0.006} & \underline{9.182}  & \textbf{0.004} & 9.914  & \underline{0.004} \\
        WFR-MFM    & \textbf{5.236} & \textbf{0.002} & \textbf{5.904} & \textbf{0.005} & \textbf{6.190} & \underline{0.006} & \textbf{6.647} & \textbf{0.001} \\
        \bottomrule
    \end{tabular}
    }
\end{table}

\begin{table}[H]
    \centering
    \caption{Comparison of method performance over time on the 100D EB dataset. }
    \label{tab:100d_eb}
   {%
    \begin{tabular}{lcccccccc}
        \toprule
        \multirow{2}{*}{\textbf{Method}} & \multicolumn{2}{c}{\textbf{t=1}} & \multicolumn{2}{c}{\textbf{t=2}} & \multicolumn{2}{c}{\textbf{t=3}} & \multicolumn{2}{c}{\textbf{t=4}} \\
        \cmidrule(lr){2-3} \cmidrule(lr){4-5} \cmidrule(lr){6-7} \cmidrule(lr){8-9}
        & $\mathcal{W}_1$ & RME & $\mathcal{W}_1$ & RME & $\mathcal{W}_1$ & RME & $\mathcal{W}_1$ & RME \\
        \midrule
        MMFM       & 11.460 & ---   & 13.879 & ---   & 14.441 & ---   & 14.907 & --- \\
        Metric FM  & 10.806 & ---   & 12.348 & ---   & 13.622 & ---   & 16.801 & --- \\
        SF2M       & 11.333 & ---   & 12.982 & ---   & 13.718 & ---   & 14.945 & --- \\
        MIOFlow    & 11.387 & ---   & 12.331 & ---   & 11.905 & ---   & 12.908 & --- \\
        TIGON      & 10.547 & 0.014 & 12.926 & 0.052 & 13.897 & 0.107 & 14.945 & 0.096 \\
        DeepRUOT   & 10.256 & \textbf{0.002} & \underline{11.103} & 0.074 & \underline{11.529} & 0.136 & \textbf{12.406} & 0.047 \\
        Var-RUOT   & 11.746 & 0.091 & 12.237 & 0.024 & 12.957 & 0.150 & 13.335 & 0.074 \\
        UOT-FM     & 10.757 & 0.056 & 12.799 & 0.037 & 13.761 & 0.044 & 15.657 & 0.022 \\
        VGFM       & 10.313 & 0.048 & 11.278 & 0.035 & 11.703 & 0.028 & \underline{12.637} & 0.066 \\
        WFR-FM     & \textbf{9.941} & 0.009 & \textbf{11.040} & \textbf{0.006} & \textbf{11.516} & \underline{0.008} & 12.664 & \textbf{0.005} \\
        WFR-MFM    & \underline{10.020} & \underline{0.007} & 11.274 & \underline{0.007} & 11.779 & \textbf{0.006} & 12.859 & \underline{0.015} \\
        \bottomrule
    \end{tabular}
    }
\end{table}

\subsection{Performance on CITE-seq Dataset}

We further evaluated our method on the CITE-seq dataset~\citep{lance2022multimodal}, comprising 31,240 cells collected over four time points. Following the preprocessing steps in~\citep{wang2025joint}, we utilized the gene expression matrix reduced to 50 dimensions via PCA. The experiments were conducted with a batch size of 1,500, setting the hyperparameters to $\delta = 30$, $p_{\text{diff}}=0.3$, and $\lambda=1$. As presented in Table~\ref{tab:cite}, although the one-step implementation (WFR-MFM) yields suboptimal performance, increasing the number of inference steps leads to significant improvements. By adjusting the inference to 10 steps, our method achieves state-of-the-art results at $t=1$ for both $\mathcal{W}_1$ and RME, and secures the best mass estimation accuracy at $t=2$. This demonstrates that while a coarse integration may be insufficient for complex dynamics, a multi-step scheme effectively unlocks the method's potential, yielding satisfactory distribution matching and population growth modeling.

\begin{table}[t]
\centering
\caption{Comparison of method performance over time on the 50D CITE dataset.}
\label{tab:cite}
{%
\begin{tabular}{lcccccc}
\toprule
\textbf{Method} & \multicolumn{2}{c}{\textbf{t=1}} & \multicolumn{2}{c}{\textbf{t=2}} & \multicolumn{2}{c}{\textbf{t=3}} \\
\cmidrule(lr){2-3} \cmidrule(lr){4-5} \cmidrule(lr){6-7}
 & $\mathcal{W}_1$ & RME & $\mathcal{W}_1$ & RME & $\mathcal{W}_1$ & RME \\
\midrule
MMFM        & 33.971 & ---   & 36.854 & ---   & 43.721 & ---   \\
Metric FM   & 28.314 & ---   & 28.617 & ---   & 33.212 & ---   \\
SF2M        & 29.543 & ---   & 32.655 & ---   & 36.265 & ---   \\
MIOFlow     & 28.290 & ---   & 28.524 & ---   & \textbf{32.230} & ---   \\
TIGON       & 28.196 & 0.186 & 27.921 & 0.545 & \underline{32.846} & 0.653 \\
DeepRUOT    & 28.245 & 0.168 & \underline{27.908} & 0.525 & 32.950 & 0.634 \\
Var-RUOT    & 30.219 & 0.331 & 32.702 & 0.325 & 40.613 & 0.486 \\
UOT-FM      & 33.531 & \underline{0.009} & 32.795 & 0.046 & 49.751 & 0.097 \\
VGFM        & 29.449 & 0.020 & 29.722 & 0.057 & 33.752 & \textbf{0.001} \\
WFR-FM      & \underline{27.831} & 0.043 & \textbf{27.478} & \underline{0.045} & 34.784 & \underline{0.022} \\

WFR-MFM & 30.841 & 0.150 & 29.390 & 0.187 & 40.471 & 0.427 \\
WFR-MFM(10 steps)& \textbf{27.509} & \textbf{0.006} & 28.255 & \textbf{0.033} & 34.055 & 0.062 \\

\bottomrule
\end{tabular}%
}
\end{table}

\subsection{Performance on Mouse Hematopoiesis Dataset}
We further validated our method on the mouse blood hematopoiesis dataset~\citep{weinreb2020lineage}, comprising 49,302 lineage-traced cells at three time points. Following PCA reduction to 50 dimensions, we applied our method with a batch size of 1,500 and fixed hyperparameters ($\delta = 15, p_{\text{diff}}=0.6, \lambda=50$). As detailed in Table~\ref{tab:mouse}, WFR-MFM yields highly competitive results: it secures the top performance for RME at $t=1$ and $\mathcal{W}_1$ at $t=2$, and remains the runner-up in other scenarios. These results highlight the method's reliability in capturing the dynamics of hematopoiesis. 

\begin{table}[H]
    \centering
    \caption{Comparison of method performance over time on the 50D Mouse dataset. }
    \label{tab:mouse}
    {%
    \begin{tabular}{lcccc}
        \toprule
        \multirow{2}{*}{\textbf{Method}} & \multicolumn{2}{c}{\textbf{t=1}} & \multicolumn{2}{c}{\textbf{t=2}} \\
        \cmidrule(lr){2-3} \cmidrule(lr){4-5} 
        & $\mathcal{W}_1$ & RME & $\mathcal{W}_1$ & RME \\
        \midrule
        MMFM       & 7.647 & ---   & 10.156 & ---   \\
        Metric FM  & 7.788 & ---   & 11.449 & ---   \\
        SF2M       & 8.217 & ---   & 11.086 & ---   \\
        MIOFlow    & 6.313 & ---   & 6.746  & ---   \\
        TIGON      & 6.140 & 0.382 & 6.973  & 0.326 \\
        DeepRUOT   & 6.052 & 0.062 & 6.757  & 0.041 \\
        Var-RUOT   & 7.951 & 0.131 & 10.862 & 0.154 \\
        UOT-FM     & 8.114 & 0.035 & 9.170  & \textbf{0.011} \\
        VGFM       & 6.274 & 0.076 & 6.796  & 0.070 \\
        WFR-FM     & \textbf{5.486} & \underline{0.012} & \underline{6.211} & \textbf{0.011} \\
        WFR-MFM    & \underline{5.925} & \textbf{0.009} & \textbf{5.548} & \underline{0.023} \\
        \bottomrule
    \end{tabular}
    }
\end{table}

\subsection{Performance on Simulation Perturbation Dataset}\label{sim perturb}

In this section, we simulate a synthetic gene regulatory network for a perturbation experiment. The dynamics of the system are governed by the following set of SDEs:
\begin{equation*}
\begin{gathered}
\frac{dX_1}{dt} = \frac{\rho_{1,c} + \alpha_1 X_1^2}{1 + \alpha_1 X_1^2 + \gamma_2 X_2^2 + \gamma_3 X_3^2} - \delta_1 X_1 + \eta_1 \xi_t, \\
\frac{dX_2}{dt} = \frac{\rho_{2,c} + \alpha_2 X_2^2}{1 + \gamma_1 X_1^2 + \alpha_2 X_2^2 + \gamma_3 X_3^2} - \delta_2 X_2 + \eta_2\xi_t, \\
\frac{dX_3}{dt} = \frac{\rho_3 + \alpha_3 X_3^2}{1 + \alpha_3 X_3^2} - \delta_3 X_3 + \eta_3 \xi_t.	
\end{gathered}
\end{equation*}

Here, \( X_i(t) \) represents the concentration of gene \( i \). The model features a toggle switch between \( X_1 \) and \( X_2 \) (mutual inhibition and self-activation), where both are further regulated by \( X_3 \) and an external signal \( \beta \). The equations are parameterized by rates for basic transcription (\( \rho_{i,\cdot} \)), self-activation (\( \alpha_i \)), inhibition (\( \gamma_i \)), and degradation (\( \delta_i \)), with an additional stochastic noise term (\( \eta_i \xi_t \)).
In addition, we incorporate probabilistic cell death, which is dependent on the expression of \( X_3 \). The instantaneous death rate \( g \) is defined as
\begin{equation*}
g = \alpha_{g,c} \frac{X_3^2}{1 + X_3^2}.
\end{equation*}
which depends on condition-specific parameter \(\alpha_{g,c}\). Each perturbation condition \(c\) corresponds to a triplet of parameters \((\rho_{1,c},\rho_{2,c},\alpha_{g,c})\), randomly sampled from predefined ranges. From the full parameter space, we generate \(5100\) perturbation conditions. A subset of \(100\) conditions is selected to form the training set, and the remaining \(5000\) are reserved for evaluation. This construction emulates realistic experimental settings where only a limited set of perturbations is experimentally measured, and the goal is to infer system behavior under a much larger number of unseen perturbations.

In this dataset, for each perturbation condition $c$, we set the WFR mass variation penalty
$\delta_c = N_c / N_{\mathrm{ctrl}}$, where $N_c$ and $N_{\mathrm{ctrl}}$ denote the numbers of cells under condition $c$ and the unperturbed condition, respectively. The proportion of cross-time sampling is fixed to $p_{\mathrm{diff}} = 0.5$, and the mass-growth regularization coefficient is set to $\lambda = 0.1$.

\end{document}